\documentclass{article}

\usepackage{microtype}
\usepackage{graphicx}
\usepackage{subcaption}
\usepackage{booktabs} %
\usepackage{hyperref}

\PassOptionsToPackage{authoryear}{natbib}

\usepackage{algorithm}
\usepackage{algpseudocode}

\newcommand{\seperator}[1][.2pt]{\par\vskip.5\baselineskip\hrule height #1\par\vskip.5\baselineskip}

\usepackage[preprint]{neurips_2025}

\usepackage[utf8]{inputenc} 
\usepackage[T1]{fontenc}    
\usepackage{hyperref}       
\usepackage{url}            
\usepackage{booktabs}       
\usepackage{amsfonts}       
\usepackage{nicefrac}       
\usepackage{microtype}      
\usepackage{xcolor}         
\usepackage{amsmath}
\usepackage{amssymb}
\usepackage{mathtools}
\usepackage{amsthm}
\usepackage{verbatim}
\usepackage{bbm}
\usepackage{thmtools, thm-restate}
\usepackage[capitalize,noabbrev]{cleveref}

\theoremstyle{plain}
\newtheorem{theorem}{Theorem}[section]
\newtheorem{proposition}[theorem]{Proposition}
\newtheorem{lemma}[theorem]{Lemma}

\theoremstyle{definition}
\newtheorem{definition}[theorem]{Definition}
\newtheorem{assumption}[theorem]{Condition}
\newtheorem*{claim*}{Claim}
\newtheorem{claim}{Claim}
\theoremstyle{remark}
\newtheorem{remark}[theorem]{Remark}
\newtheorem*{remark*}{Remark}
\newcommand\numberthis{\addtocounter{equation}{1}\tag{\theequation}}

\title{A Theoretical Characterization of Optimal Data Augmentations in Self-Supervised Learning}

\author{%
  Shlomo Libo Feigin, Maximilian Fleissner, Debarghya Ghoshdastidar \\
   School of Computation, Information and Technology \\
  Techincal University of Munich\\
  \texttt{\{s.libo, m.fleissner, d.ghoshdastidar\}@tum.de} \\
}

\begin{document}

\maketitle

\begin{abstract}
 Data augmentations play an important role in the recent success of self-supervised learning (SSL).
While augmentations are commonly understood to encode invariances between different views into the learned representations, this interpretation overlooks the impact of the pretraining architecture and suggests that SSL would require diverse augmentations which resemble the 
data to work well. However, these assumptions do not align with empirical evidence, encouraging further theoretical understanding to guide the principled design of  augmentations in new domains. To this end, we use kernel theory to derive analytical expressions for data augmentations that achieve desired target representations after pretraining. We consider non-contrastive and contrastive losses, namely VICReg, Barlow Twins and the Spectral Contrastive Loss, and provide an algorithm to construct such augmentations. Our analysis shows that augmentations need not be similar to the data to learn useful representations, nor be diverse, and that the architecture has a significant impact on the optimal augmentations.
\end{abstract}

\section{Introduction}

Self-supervised learning (SSL) has gained prominence in recent years, serving as one of the backbones of the foundation models driving current progress in artificial intelligence. Instead of using labels, SSL employs a surrogate objective to learn representations, which are then used for downstream tasks. Joint embedding methods \citep{chen_simple_2020, he_momentum_2020, BYOL,zbontar_barlow_2021} in particular have seen a rise in popularity, achieving competitive performance with supervised representation learning for downstream classification, image segmentation and object detection \citep{chen2020exploringsimplesiameserepresentation, bardes_vicreg_2022}. Broadly speaking, these models encourage different views of the same underlying point to align closely in the embedding space, while preventing the representations from collapsing \citep{jing_understanding_2022}. In the vision domain, the views are usually given by augmentations such as random crop, Gaussian blur, and color distortion \citep{chen_simple_2020}.

The choice of augmentations is critical to the representations learned by the SSL objective. Different downstream tasks may require different augmentation; for example, \citet{purushwalkam} show that cropping encourages invariance to occlusions but negatively affects downstream tasks that require category and viewpoint invariance. In addition, \citet{shouldnt_be_contrastive, DBLP:journals/corr/abs-2111-09613, zhang2022rethinkingaugmentationmodulecontrastive,bendidi2023freelunchselfsupervised} demonstrate that certain augmentations benefit or penalize downstream classification on different classes, and \citet{ericsson2022selfsupervisedmodelstransferinvestigating} empirically show that pose-related tasks and classification-related tasks benefit from opposite augmentations. \textit{Previous theoretical studies largely do not address the subtleties of augmentation choice}. Instead, by assuming a certain relationship between the data and the augmentations, they provide guarantees specifically on downstream classification tasks \citep{arora_theoretical_2019, haochen_provable_2022, saunshi_understanding_2022}. 

In practice, however, the choice of suitable augmentations depends not only on the downstream task, but also intricately on the domain \citep{bendidi2023freelunchselfsupervised, balestriero2023cookbookselfsupervisedlearning}. Medicine still struggles to benefit from the success of SSL, arguably since assumptions that hold for natural images do not hold for medical images \citep{huang_self-supervised_2023}. As an example, consider brain scans. Unlike natural images, which typically include a central object, cropping parts of an image with a tumor can drastically change the interpretation of the scan. In practice, augmentations are therefore carefully crafted for different applications, a manual and empirically driven process. Theoretical insights are scarce, and as we discuss next, the few that exist do not explain empirical phenomena well. For example, if augmentations primarily serve the role of generating different views, one would expect that data augmentations need to (a) be similar to the original data, and (b) capture a diverse set of views. Neither holds in practice. Firstly, strong data augmentations, such as random crop and cutout, have drastically different marginal distributions than the original data \citep{gontijolopes2020affinitydiversityquantifyingmechanisms}. At the same time, these are exactly the augmentations that empirically provide the most benefit for downstream performance \citep{chen_simple_2020}. In contrast, natural-seeming augmentations, such as adding Gaussian noise, only show limited benefit. 
Secondly, \citet{bardes_vicreg_2022, pmlr-v202-cabannes23a} find that it is always better to use more data instead of more augmentations. In fact, \citet{chen_simple_2020} use only three types of augmentations. Finally, \citet{moutakanni2024dontneeddataaugmentationselfsupervised} stretch both of the above assumptions to the limit, showing that with just one augmentation, cropping, one can achieve state-of-the-art representations provided sufficient data.

\begin{figure}[!t]
  \centering
  \hspace{-0.07\textwidth}
  \begin{subfigure}[t]{0.22\textwidth}
  
    \centering
    \includegraphics[height=5cm]{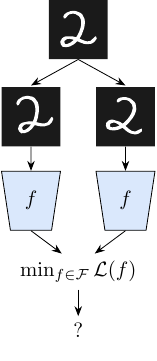}
    \caption{Previous works}
    \label{fig:prev}
  \end{subfigure}\hspace{-0.07\textwidth}
  \begin{minipage}[t]{0.22\textwidth}
    \centering
    \includegraphics[height=5cm]{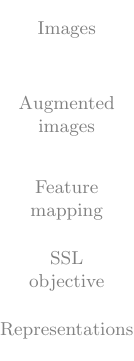}
  \end{minipage}\hspace{-0.07\textwidth}
  \begin{subfigure}[t]{0.22\textwidth}
    \centering
    \includegraphics[height=5cm]{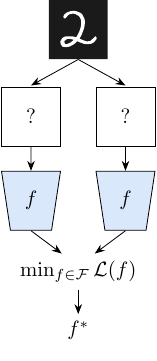}
    \caption{Our study}
    \label{fig:ours}
  \end{subfigure}\hspace{0.0\textwidth}
  \begin{subfigure}[t]{0.45\textwidth}
    \centering
    \includegraphics[height=5cm]{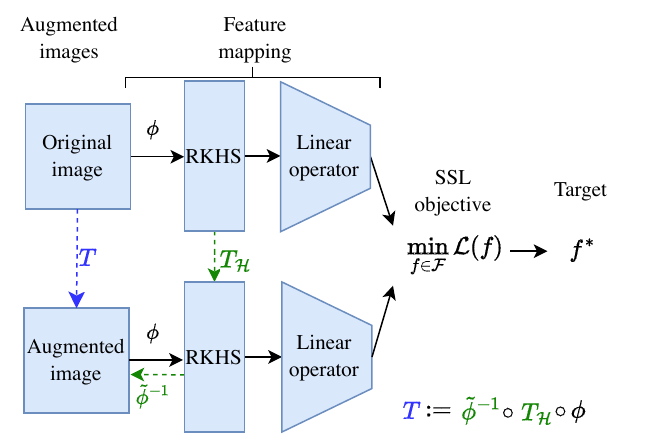}\caption{Our approach}
    \label{fig:our_approach}
  \end{subfigure}

  \caption{
    (a) Previous theoretical works assume certain augmentation characteristics and examine the learned representations.
    \quad (b) Our study asks the reverse question: given target representations (e.g.\ from a pretrained ResNet), what augmentations achieve them?
    \quad (c) An illustration of our pipeline: we find a transformation $T_{\mathcal{H}}$ in the RKHS that yields the target representations $f^*$ (Theorems~\ref{main lemma VICReg},~\ref{main lemma SCLNorm},~\ref{main lemma Barlow Twins}), then translate $T_{\mathcal{H}}$ back to input space by solving a pre-image problem $\tilde\phi^{-1}$ (Algorithm~\ref{algorithm}).
  }
  \label{fig:combined-all}
\end{figure}

In self-supervised learning, a unique interplay exists between data, augmentations and representations \citep{pmlr-v202-cabannes23a}. Prior works predominantly focus on one side of this interaction, asking how data and augmentations influence the learned representations. In this paper, we study SSL from a different angle. We ask: \textit{For given data and desired target representations, what augmentations result in these representations?} This conceptually deviates from prior studies, as illustrated in Figure \ref{fig:combined-all}.

\paragraph{Contributions.} To tackle this question, we place ourselves within existing theoretical frameworks for SSL. We formally prove that suitable data augmentations together with a sufficiently expressive hypothesis class can learn any desired representation by joint embedding methods such as VICReg \citep{bardes_vicreg_2022}, Barlow Twins \citep{zbontar_barlow_2021} and the Spectral Contrastive Loss \citep{haochen_provable_2022} --- the latter being a theoretical proxy to SimCLR \citep{chen_simple_2020}. We derive analytical expressions for the augmentations in Section \ref{main_sec}, and propose an algorithm to compute these augmentations in Section \ref{algorithm_sec}. Finally, we give insights about augmentation choice in Section \ref{discussion_sec}, interpreting the aforementioned empirical phenomena through the lens of our analysis. To summarize, our main contributions are the following.
\begin{itemize} \item  We prove that for VICReg, Barlow Twins and the Spectral Contrastive Loss, it is possible to guarantee the recovery of any representation for the input data, given suitable augmentations and a sufficiently expressive function class. 
\item For VICReg and the Spectral Contrastive Loss, we derive a closed-form solution for the augmentations; for Barlow Twins, the augmentations are expressed
through a solution of a continuous-time Lyapunov equation. To the best of our knowledge, this is the first method to construct augmentations explicitly for any given target representation.
\item Our theoretical results provide new insights into the role of augmentations in SSL: (a) Augmentations need not be similar to the original data. (b) Even very few augmentations can provide good representations, as empirically observed by \citealt{moutakanni2024dontneeddataaugmentationselfsupervised}. (c) Augmentations corresponding to the same representations can be recognizably different depending on the architecture. (d) Augmentations can act as projections in the feature space as opposed to different views of the data.

\end{itemize}

\section{Related Work}\label{related work}
\paragraph{Theory of Self-Supervised Learning.}

Previous theoretical works on self-supervised learning primarily focus on mathematically describing the learned representations, while implicitly assuming that useful augmentations are already given. This high-level idea has been formalized in numerous ways. Positive pairs can be assumed to be independent samples from the same class \citep{arora_theoretical_2019}, or a data augmentation graph is defined in which positive pairs are connected \citep{haochen_provable_2022}. Furthermore, a suitable integral operator can encode similarity between positive pairs \citep{pmlr-v202-cabannes23a}, or the target representations are assumed to be contained in a Hilbert space defined by the augmentation \citep{zhai2024understanding}. In this work, we turn this question around: Given target representations, what are the augmentations one needs to learn these representations?

\paragraph{Kernel Methods and SSL.} Studying supervised deep learning through kernel methods has proven to be fruitful, most notably by virtue of the neural tangent kernel (NTK) \citep{jacot2018neural}. This has sparked several works that assume kernel models for the representation function in SSL \citep{ kiani_joint_2022, NEURIPS2022_aa56c745, pmlr-v202-simon23a,pmlr-v202-cabannes23a, esser2024non}. Indeed, the validity of the NTK approximation in SSL has recently been proven for Barlow Twins \citep{fleissner2024infinite}, justifying the use of kernel theory to understand SSL with neural networks.
Additionally, it is known that contrastive learning can in fact be viewed as kernel learning \citep{johnson2022contrastive}, which has been used to derive new generalization error bounds for SSL \citep{zhai2024understanding}. In this paper, we therefore also consider a kernel setting. For VICReg our framework coincides with that of \citet{NEURIPS2022_aa56c745}, for Barlow Twins our framework matches \citet{pmlr-v202-simon23a}, and for the Spectral Contrastive Loss our framework is similar to that of \citep{esser2024non}.

\paragraph{The Role of Data Augmentations.} The role of data augmentations has been studied both in the supervised setting as well as in the self-supervised setting. In the supervised setting,  \citet{gontijolopes2020affinitydiversityquantifyingmechanisms,kim2022what} suggest that augmentations should be similar to the original data distribution, and diverse. \citet{group} view augmentations as group actions that keep the data distribution approximately similar. \citet{geiping2023how} show that augmentations can be helpful even if they encode invariances not present in the original data. In SSL, \citet{wang2020understanding} explain the role of data augmentations through the alignment of the positive samples and the uniformity of induced distribution. \citet{wang2022chaos} present the role of augmentations as connecting the data through overlaps, which then cause class-separated representations in downstream classification. \citet{content_style} view augmentations as transformations that preserve the semantic content of an image while changing the style component. Taking an information theoretic perspective, \citet{tian_what_2020} argue that augmentations should reduce the mutual information between views as much as possible while keeping the shared task-relevant information intact.
It is challenging (and in some cases even impossible) to reconcile these interpretations with the empirical observation that augmentations produce useful representations even when they are not diverse or similar to the original distribution \citep{moutakanni2024dontneeddataaugmentationselfsupervised}.

\section{Preliminaries}\label{preliminaries}

\subsection{Problem Statement and Approach}
\paragraph{Joint Embedding Loss Functions.}\label{Joint Embedding Loss} Throughout this paper, we assume we are given $n$ data points $\{x_i\}_{i=1}^{n}\subseteq\mathcal{X} \subseteq \mathbb{R}^m$. For each of these points $x_i$, we create two views by sampling from a random augmentation map $T:\mathcal{X}\rightarrow\mathcal{X}$ with distribution $\mathcal{T}$. This yields $T_i(x_i)$ and $T'_i(x_i)$. Both views are passed to a function $f: \mathcal{X} \rightarrow \mathbb{R}^d$ that maps to a lower-dimensional space with $d < \min(n,m)$. In practice, $f$ is typically a neural network, while theoretical works often consider $f$ to be a kernel function \citep{ kiani_joint_2022, NEURIPS2022_aa56c745, pmlr-v202-simon23a,pmlr-v202-cabannes23a, esser2024non}. We denote $\mathcal{F}$ for the function class from which $f$ is chosen. Denoting $z_i = f(T(x_i))$ and $z_i' = f(T'(x_i))$, a joint embedding loss $L(Z,Z')$ is computed on the matrices $Z=[z_1,\ldots,z_n]$ and $Z'=[z_1',\ldots,z_n']\in \mathbb{R}^{d\times n}$. The loss minimized over $\mathcal{F}$ is of the form 
    $\mathcal{L}\Bigl(\{x_i\}_{i=1}^{n},\mathcal{T},f\Bigl)=\mathbb{E}_{Z,Z'}\Bigl[L(Z,Z')\Bigl]$,
where the expectation is with respect to the randomness in the augmentations. Depending on the method, $T_i$ and $T'_i$ can either be independent samples from the underlying distribution $\mathcal{T}$, or they can be conditioned on the event $T'_i\neq T_i$ (in that case, augmentations always yield two distinct views of the same point $x_i$). .

\paragraph{Kernels, RKHS and Function class.}
We assume that we are given a kernel $\kappa: \mathcal{X} \times \mathcal{X} \rightarrow \mathbb{R}$. The canonical feature map of this kernel $\kappa$ is denoted by $\phi(x):\mathcal{X}\rightarrow\mathcal{H}$, where $\mathcal{H}$ is the reproducing kernel Hilbert space (RKHS) associated with $\kappa$.
From there, we use the tensor product notation $\mathbb{R}^d\otimes\mathcal{H}$ to denote the set of linear maps from $\mathcal{H}$ to $\mathbb{R}^d$. An element $\Theta \in \mathbb{R}^d\otimes\mathcal{H}$ can be thought of as a matrix with $d$ rows that are elements of $\mathcal{H}$. The adjoint operator, denoted as $\Theta^\top  \in \mathcal{H}\otimes\mathbb{R}^d$, can be thought of as the transposed matrix. We equip the function space with the Hilbert-Schmidt norm on $\mathbb{R}^d\otimes\mathcal{H}$. The space $\mathcal{F}$ from which our model $f$ is chosen is given by the set of all functions $x \mapsto \Theta^\top  \phi(x)$ where $\Theta \in \mathbb{R}^d\otimes\mathcal{H}$, again equipped with the Hilbert-Schmidt norm on $\Theta$. In this paper, we restrict ourselves to the least norm solutions of $\mathcal{L}\left(\{\phi(x_i)\}_{i=1}^n, \mathcal{T}, f\right)$ on $\mathcal{F}$, as is commonly done in kernel methods. Formally:
\begin{definition}[Least Norm Minimizers]
A function $\hat f$ is said to minimize a loss $\mathcal{L}$ with least norm over $\mathcal{F}$ if $\mathcal{L}(\hat f) = \inf\limits_{f \in \mathcal{F}} \mathcal{L}(f)$ and $\|\hat f\|_\mathcal{F} \le \| f \|_\mathcal{F}$ for all $f$ with $\mathcal{L}(f) = \inf\limits_{g \in \mathcal{F}} \mathcal{L}(g)$.
\end{definition}

\paragraph{Problem Statement.} 
Given data $\{x_i\}_{i=1}^{n}\subset\mathcal{X}$ and a desired target function $f^\ast:\mathcal{X}\rightarrow\mathbb{R}^d$, our goal is to find a distribution $\mathcal{T}$ of random transformations such that $f^*$ is equivalent to $\min_{f\in\mathcal{F}}\mathcal{L}(\{x_i\}_{i=1}^n, \mathcal{T},f)$. We consider two functions $f$ and $g$ to be equivalent if they are identical up to an invertible affine transformation. This notion of equivalence is justified by \citet[Lemma 3.1]{haochen_provable_2022}, who prove that invertible affine transformations do not influence the downstream performance when using the standard linear evaluation protocol \citep{chen_simple_2020}.

Motivated by the empirical study of \citet{moutakanni2024dontneeddataaugmentationselfsupervised}, we consider distributions $\mathcal{T}$ that can lead to only two possible augmentations --- the identity and some other transformation  $T$. We construct $T$ by first finding a suitable augmentation in the Hilbert space $\mathcal{H}$, and only later translating it back to the input space $\mathcal{X}$. Given a target representation $f^*$, we first identify a binary distribution $\mathcal{T}_{\mathcal{H}}$ that either samples from the identity operator in $\mathcal{H}$ or reduces to a transformation $T_{\mathcal{H}}$. The distribution $\mathcal{T}_\mathcal{H}$ is carefully chosen in a way that ensures the least-norm minimizer of the loss is equivalent to the desired target $f^*$. Having found the correct augmentation in the Hilbert space, we translate it back to the input space by defining $T\coloneqq \tilde{\phi}^{-1}\circ T_{\mathcal{H}}\circ\phi$
where $\tilde{\phi}^{-1}$ is a solution to the pre-image problem for kernel machines. Our approach is illustrated in Figure \ref{fig:our_approach}.

\paragraph{Continuous-Time Lyapunov Equations.} Our theoretical results for the Barlow Twins loss partly build on solutions to continuous-time Lyapunov equations. These are matrix equations of the form $AX+XA^\top =C$, where $X,A,C$ are matrices of appropriate size. For our purposes, the main result we need is that if $C$ is symmetric and $A$ is a positive-definite matrix, then there is a unique symmetric solution $X$ that solves the Lyapunov equation \citep[Theorem 6.4.2]{ortega_other_1987}. The solution $X$ can be represented analytically in terms of $C$ and $A$.

\paragraph{Additional Notation.} For points $\{x_i\}_{i=1}^n$ we denote by $\Phi\in\mathcal{H}\otimes\mathbb{R}^n$ the operator $[\phi(x_1), \ldots, \phi(x_n)]$. $\Phi$ can be thought of as the transpose of the design matrix.
We use $\mathcal{D}\otimes\mathcal{D}$ to denote the product of two probability measures $\mathcal{D}$. $I_\mathcal{H}$ is the identity operator. $\circ$ is the composition operator. The centering matrix is defined as $H_n=I_n - \frac{1}{n} \mathbf{1}_n \mathbf{1}_n^\top \in \mathbb{R}^{n \times n}$, where $\mathbf{1}_n$ is a vector of ones.  We omit subscripts where the dimension is clear from the context. We denote the sample covariance matrix by $\text{cov}(X) = \frac{1}{n}(XH)(XH)^\top $. Finally, the equivalence relation $f\overset{\text{aff}}{\sim}g\big\rvert_{S}$ means there exist an invertible matrix $A$ and a vector $b$ such that $\forall_{x\in S}{f(x)=Ag(x)+b}$; when $S$ is the entire range of $g$ and $f$ we simply denote $f\overset{\text{aff}}{\sim}g$. 

\paragraph{Assumptions.} In this paper, we assume two conditions to be satisfied. Firstly, the target representations should have full rank. This condition is reasonably mild: if $f^*: \mathbb{R}^m\rightarrow\mathbb{R}^d$ has linearly dependent dimensions in its range, then there exists an equivalent $f': \mathbb{R}^m \rightarrow \mathbb{R}^{d'}
$ with independent dimensions such that $f^*$ and $f'$ will perform the same under the standard linear evaluation protocol. 

\begin{assumption}\label{non_redundant_representation} Given a target function $f^*$ and data $\{x_i\}_{i=1}^{n}$, the covariance matrix 
$\text{cov}\left([f^\ast(x_1), \ldots, f^\ast(x_n)]\right)$ has full rank $d$.
\end{assumption}

Secondly, the Gram matrix of the kernel should be invertible. For universal kernels (e.g. the Gaussian RBF kernel), Condition \ref{invertable_gram} holds for any set of distinct samples $\{x_i\}_{i=1}^{n}$.

\begin{assumption}\label{invertable_gram}
Given a kernel function $\kappa(\cdot, \cdot)$ and data $\{x_i\}_{i=1}^{n}$,  $K=[\kappa(x_i, x_j)]_{i,j}$ has full rank.
\end{assumption}

\subsection{Non-Contrastive and Contrastive Loss Functions}
We present the loss functions we analyze. All claims made are proven in the appendix. 

\paragraph{VICReg.} The VICReg loss \citep{bardes_vicreg_2022} is defined via three components:
\begin{gather*}
\mathcal{L}_{\text{VIC}}\Bigl(\{x_i\}_{i=1}^{n}, \mathcal{T},f\Bigl)=\mathbb{E}_{Z,Z'}\Bigl[\lambda s(Z, Z') + \mu [ v(Z) + v(Z')] + \nu [ c(Z) + c(Z')]\Bigl] \text{, where}
\\
s(Z,Z') = \frac{1}{n}\sum_{i=1}^{n}\Vert z_{i}-z_{i}' \Vert_{2}^{2},
\quad
v(Z) = \frac{1}{d}\sum_{i=1}^{d}\left(1-[\text{cov}\left(Z\right)]_{i,i}\right)^{2},
\quad
c(Z) = \frac{1}{d}\sum_{i\neq j}^{d}[\text{cov}(Z)]_{i,j}^2
\end{gather*}

and $\lambda,\mu, \nu > 0$ are hyper-parameters. $z_i=f\Bigl(T_i(x_i)\Bigl),z_i'=f\Bigl(T'_i(x_i)\Bigl)$, where $\left(T_i, T'_i\right) {\sim} \mathcal{T}\otimes\mathcal{T}$.

This definition is identical to the original VICReg loss \citep{zbontar_barlow_2021}, except we use the $L_2$ loss instead of the hinge loss in the definition of $v$. To be consistent with previous theoretical works \citep{NEURIPS2022_aa56c745,pmlr-v202-cabannes23a}, we define $v$ using the variance instead of the standard deviation, however our results still hold when using the standard deviation (Appendix \ref{appendix_def_vicreg}).

\paragraph{Barlow Twins.} The Barlow Twins Loss \citep{zbontar_barlow_2021} is defined in the following way:
\begin{align*}
    \mathcal{L}_{\text{BT}}\Bigl(\{x_i\}_{i=1}^{n},\mathcal{T},f\Bigl)=\mathbb{E}_{Z,Z'}\Bigl[\sum_i^n (1- \mathcal{C}_{ii})^2+ \lambda\sum_{i\neq j} (1- \mathcal{C}_{ij})^2\Bigl];\quad \mathcal{C}=\frac{1}{2n}(ZZ'^\top  + Z'Z^\top )
\end{align*}
Where $\lambda>0$ is a hyperparameter. 
$z_i=f\Bigl(T_i(x_i)\Bigl),z_i'=f\Bigl(T'_i(x_i)\Bigl)$, where the random augmentations $\left(T_i, T_i'\right) {\sim} \mathcal{T}\otimes\mathcal{T}$ are conditioned on the event $T_i\neq T'_i$. Similar to \citet{pmlr-v202-simon23a}, we consider a symmetrized version of the unnormalized cross-correlation for $\mathcal{C}$.

\paragraph{Spectral Contrastive Loss.} The Spectral Contrastive Loss (SCL) \citep{haochen_provable_2022} is a theoretical proxy to the SimCLR loss. We define the following sample loss:
\begin{align*}
\mathcal{L}_{\text{SCL}}\Bigl(\{x_i\}_{i=1}^{n},\mathcal{T},f\Bigl)&=\mathbb{E}_{Z,Z'}\Bigl[-\frac{2}n{}\sum_{i=1}^nz_i^\top z'_i+\frac{1}{n^2}\sum_{i\neq j}(z_i^\top z'_j)^2 + \frac{1}{2n^2}\sum_{i=1}^n( \Vert z_i\Vert^4+\Vert z_i'\Vert^4)\Bigl]
\end{align*}
$z_i=f\Bigl(T_i(x_i)\Bigl),z_i'=f\Bigl(T'_i(x_i)\Bigl)$, where $\left(T_i, T'_i\right) {\sim} \mathcal{T}\otimes\mathcal{T}$.

\begin{remark*}
    Without regularization, the term  $-\frac{2}n{}\sum_{i=1}^nz_i^\top z'_i+\frac{1}{n^2}\sum_{i\neq j}(z_i^\top z'_j)^2$ alone can diverge to $-\infty$ for small batches, as we show in the appendix. Previous definitions of the sample loss for SCL deal with that fact in various ways. For example,  \citet{haochen_provable_2022} enforce the norm of representations to be fixed in their experiments, while \citet{esser2024non} regularize using the norm $\Vert\cdot\Vert_{\mathcal{H}}$. The regularization term $\frac{1}{2n^2}\sum_{i=1}^n( \Vert z_i\Vert^4+\Vert z_i'\Vert^4)$ is theoretically motivated by a guaranteed tight bound $\mathcal{L}_{\text{SCL}}\geq-d$ for any $n$, which we prove in the appendix. As $n\rightarrow\infty$, the regularization decreases and $\mathcal{L}_{\text{SCL}}$ uniformly converges to the population SCL as defined in \citet{haochen_provable_2022}.    
\end{remark*}

\section{Main Results}\label{main_sec}

In this section, we develop our main results for VICReg (Theorem \ref{main lemma VICReg}), the Spectral Contrastive Loss (Theorem \ref{main lemma SCLNorm}) and Barlow Twins (Theorem \ref{main lemma Barlow Twins}), proving that any desired target representations can be found by using suitable augmentations, up to equivalence. We defer the proofs to the appendix. We begin by defining the optimal augmentations for VICReg and SCL, and then verify their optimality.

\begin{definition}[Optimal VICReg and SCL Augmentations]\label{definition vicreg} Consider a target representation of the form $f^* = C \Phi^\top $, where $C \in \mathbb{R}^{d\times n}$ has full rank. Assuming Condition \ref{invertable_gram}, we define $\mathcal{T}_\mathcal{H}(C)$ as a distribution of transformations yielding $I_\mathcal{H}$ and $\Phi\,C^\top \bigl(C\,K\,C^\top \bigr)^{-1}C\,\Phi^\top $ with probability $\frac{1}{2}$ each.
\end{definition}

\begin{restatable}[Optimality of Augmentations for VICReg]{theorem}{mainLemmaVICReg}\label{main lemma VICReg}
Let $f^* = C\,\Phi^\top , C \in \mathbb{R}^{d\times n}$ satisfy Condition \ref{non_redundant_representation} and assume Condition \ref{invertable_gram}.  
Then, $C$ has full rank and any $f$ that is a least norm minimizer of 
$\mathcal{L}_{\text{VIC}}\Bigl(\{\phi(x_i)\}_{i=1}^n,\;\mathcal{T}_\mathcal{H}(C),\;f\Bigr)$
over $\mathbb{R}^d \otimes \mathcal{H}$ satisfies $f\overset{\text{aff}}{\sim}f^*$.
\end{restatable}

Interestingly, the exact same augmentations can be used for SCL:
\begin{restatable}[Optimality of Augmentations for SCL]{theorem}{mainLemmaSCLNorm}\label{main lemma SCLNorm}
Let $f^* = C\,\Phi^\top , C \in \mathbb{R}^{d\times n}$ satisfy Condition \ref{non_redundant_representation} and assume Condition \ref{invertable_gram}.  
Then, $C$ has full rank and any $f$ that is a least norm minimizer of 
$\mathcal{L}_{\text{SCL}}\Bigl(\{\phi(x_i)\}_{i=1}^n,\;\mathcal{T}_\mathcal{H}(C),\;f\Bigr)$
over $\mathbb{R}^d \otimes \mathcal{H}$ satisfies $f\overset{\text{aff}}{\sim}f^*$.
\end{restatable}

Notably, the transformation $T=\Phi\,C^\top \bigl(C\,K\,C^\top \bigr)^{-1}C\,\Phi^\top $ has rank $d$ and $T^2=T$, whereas the data spans a subspace of dimension $n>d$ in $\mathcal{H}$ as a consequence of Condition \ref{invertable_gram}. Thus, $T$ can be interpreted as a projection to a low dimensional subspace in the feature space $\mathcal{H}$. In particular, the augmentations induced by $\mathcal{T}_\mathcal{H}$ will have a different marginal distribution than the data, no matter what $f^*$ is. This is contrary to previous interpretations (discussed in Section \ref{related work}), which argue that data augmentations should produce views similar to the data to learn a ``good'' $f^*$.

Our results for Barlow Twins take a similar form. We again state the optimal augmentations first.

\begin{definition}[Optimal Barlow Twins Augmentations]\label{definition bt} Given a rank $d$ matrix $C \in \mathbb{R}^{d\times n}$ and assuming Condition \ref{invertable_gram}, we define  $\mathcal{T}_\mathcal{H}^\text{BT}(C)$ as a distribution of transformations yielding $I_\mathcal{H}$ and $\Phi K^{-\frac{1}{2}}B K^{-\frac{1}{2}}\Phi^\top $ with probability $\frac{1}{2}$ each, where $B$ is the unique solution to the continuous-time Lyapunov equation
$KB + BK^\top =2n \cdot K^{\frac{1}{2}} C^\top \Bigl(CKC^\top \Bigr)^{-2}CK^{\frac{1}{2}}$.

\end{definition}
\begin{restatable}[Optimality of Augmentations for Barlow Twins]{theorem}{mainLemmaBt}\label{main lemma Barlow Twins}
Let $f^* = C\,\Phi^\top , C \in \mathbb{R}^{d\times n}$ satisfy Condition \ref{non_redundant_representation} and assume Condition \ref{invertable_gram}. 
Then, $C$ has full rank and any $f$ that is a least norm minimizer of $\mathcal{L}_{\text{BT}}\Bigl(\{\phi(x_i)\}_{i=1}^n,\;\mathcal{T}_\mathcal{H}^\text{BT}(C),\;f\Bigr)$
over $\mathbb{R}^d \otimes \mathcal{H}$ satisfies $f\overset{\text{aff}}{\sim}f^*$.
\end{restatable}

\begin{remark*}
    Assuming the functional form $f^*=C\Phi^\top $ for the target representation is not a restrictive condition. By virtue of the representer theorem \citep{scholkopf_generalized_2001}, any least norm minimizer of a loss function that only depends on the training data is certainly contained in the span of $\{\phi(x_i)\}_{i=1}^n$. 
    For a general function $f^*:~\mathcal{X}\rightarrow\mathbb{R}^d$, for example a pretrained ResNet, Corollary~\ref{main proposition} below allows exactly reconstructing representations $\{f^*(x_i)\}_{i=1}^n$ on any set of training data. 
\end{remark*}

\begin{restatable}[Reconstruction of General Representations]{corollary}{mainProposition}\label{main proposition}
Let $(\mathcal{L}, \mathcal{T})\in \Bigl\{(\mathcal{L}_{\text{SCL}}, \mathcal{T}_\mathcal{H}), (\mathcal{L}_{\text{VIC}}, \mathcal{T}_\mathcal{H}), (\mathcal{L}_{\text{BT}}, \mathcal{T}_\mathcal{H}^\text{BT})\Bigr\}$ and $f^*:\mathcal{X}\rightarrow\mathbb{R}^d$ satisfy Condition \ref{non_redundant_representation} on data $\{x_i\}_{i=1}^n$ satisfying Condition \ref{invertable_gram}. Define $F=\Bigl[f^*(x_1),\ldots, f^*(x_n)\Bigr]$. Then, any $f$ that is a least norm minimizer of 
$\mathcal{L}\Bigl(\{\phi(x_i)\}_{i=1}^n,\;\mathcal{T}_\mathcal{H}(FK^{-1}),\;f\Bigr)$
over $\mathbb{R}^d \otimes \mathcal{H}$ satisfies 
$f\circ\phi\overset{\text{aff}}{\sim}f^*\big\rvert_{\{x_i\}_{i=1}^n}$
\end{restatable}

\section{Algorithm}\label{algorithm_sec}

Based on our theoretical results, we provide Algorithm \ref{algorithm} to compute the augmentations in the input space given an arbitrary target $f^*$ (e.g. a pretrained ResNet or a ViT). Below, we discuss the preimage problem, matching $f^*$ on new data, and how the algorithm can be scaled to large datasets.

\begin{algorithm}[!t]
\caption{Augmentation Learning for SSL}\label{algorithm}
\begin{algorithmic}[1]
\Require Data $\{x_i\}_{i=1}^n$, target function $f^*$, kernel function $\kappa(\cdot, \cdot)$, data to augment $\{\hat{x}_i\}_{i=1}^k$, ridge parameter $\lambda_{\text{ridge}}$, $\text{method}\in\{{\text{VICReg}, \text{Barlow Twins},\text{SCL}}\}$
\Ensure Augmented data $\{\hat{x}'_i\}_{i=1}^k$
\seperator

\State $K\leftarrow$ Kernel matrix on $\{x_i\}_{i=1}^n$; \quad $K_{X,\hat{X}}\leftarrow$ Cross-kernel matrix between $\{x_i\}_{i=1}^n$ and $\{\hat{x}_i\}_{i=1}^k$
\State $C \leftarrow [f^*(x_1),\ldots,f^*(x_n)](K + \lambda_{\text{ridge}} I)^{-1} $\Comment{Solve KRR to obtain representer coefficients.}
\If{method \textbf{is} VICReg \textbf{or} SCL}
\State $C_{aug} \leftarrow C^{T}\left(CKC^{T}\right)^{-1}CK_{X,\hat{X}}$ \Comment{Compute augmentation coefficients (Thms. \ref{main lemma VICReg}, \ref{main lemma SCLNorm}).}
\EndIf
\If{method  \textbf{is} Barlow Twins}
\State $B \leftarrow$ Solution of the Lyapunov equation $KB + BK^\top =2n K^{\frac{1}{2}}C^\top (CKC^\top )^{-2}CK^{\frac{1}{2}}$
\State $C_{aug} \leftarrow K^{-\frac{1}{2}}B K^{-\frac{1}{2}} K_{X,\hat{X}}$ \Comment{Compute augmentation coefficients (Thm. \ref{main lemma Barlow Twins}).} 
\EndIf
\State \Return solution of the pre-image problem for $\phi'=\Phi C_{aug}$.
\end{algorithmic}
\end{algorithm}

\paragraph{The Pre-Image Problem in Kernel Machines.} 
In the previous section, we derived $T_\mathcal{H}$ as maps in the RKHS $\mathcal{H}$, taking the form of $\phi(x_i) \mapsto \Phi M \Phi^\top  \phi(x_i)$ for a matrix $M \in \mathbb{R}^{n \times n}$. In particular, this construction yields augmented versions of original points $\phi(x)$ that are of the form $\Phi \theta$ for some $\theta \in \mathbb{R}^{n}$, and as such lie in the Hilbert space. In this section, we translate these augmentations back to the input space by identifying $x' \in \mathcal{X}$ such that $\phi(x') \approx \Phi \theta$. In general, such $x'$ need not exist \citep{NIPS1998_preimage} --- finding an approximation is a task known as the pre-image problem for kernel machines \citep{1353287}. For our purposes, we use the closed-form approximation proposed by \citet{honeine_closed-form_2011}, which we detail in the appendix. We emphasize that while our augmentations $T_\mathcal{H}$ are formally defined in the Hilbert space, the kernel trick allows all computations to happen in the input space. Moreover, for new data points $x'$, the functional form of the augmentations allows evaluating the transformation of $\phi(x')$.

\paragraph{Generalization on New Data.}
As mentioned in the previous section, our method allows exactly matching any desired target representation $f^*$ on given training data $\{x_i\}_{i=1}^n$ up to equivalence (i.e. up to an invertible affine transformation). It is natural to question the statistical soundness of interpolating on target representations $f^*(x_i)$ obtained e.g. from a ResNet. Therefore, we also incorporate the possibility of using a ridge parameter $\lambda>0$ to avoid overfitting.
We solve 
$\arg\min_{f\in\mathcal{F}}\frac{1}{n} \sum_{i=1}^n \left\| f^*(x_i)-  f(x_i) \right\|^2 + \lambda\Vert f\Vert_\mathcal{F}^2$ 
where $\mathcal{F}$ is, as before, the space of functions of the form $x \mapsto \Theta^\top  \phi(x), \Theta \in \mathbb{R}^d \otimes \mathcal{H}$ equipped with the Hilbert-Schmidt norm on $\Theta$. 
The optimal solution is $
f=F(K + \lambda_{\text{ridge}} I)^{-1}\Phi^\top $, where we denote  $F=\Bigl[f^*(x_1),\ldots,f^*(x_n)\Bigr]$.
Crucially, this is still a function of the form $f = C \Phi^\top $, and so we can use the augmentations $\mathcal{T}_\mathcal{H}\Bigl(F(K + \lambda_{\text{ridge}} I)^{-1}\Bigr)$ defined in Definition \ref{definition vicreg} and \ref{definition bt} respectively to obtain augmentations that achieve the desired target representations --- up to some small error that is introduced by $\lambda_{\text{ridge}} >0$.

\paragraph{Scalability.} Much work has been done to make kernel methods scalable. Of particular importance are random feature methods \citep{random_features} and the Nyström method \citep{Nystrom}. Kernel ridge regression has been adapted to utilize GPU hardware \citep{meanti_kernel_2020,meanti_efficient_2022}. These methods propose approximators of the form $f=C[\phi(x_1),\ldots ,\phi(x_{n'})]^\top $ where where $C$ is calculated as an efficient solution to KRR for $n$ points, $n'<n$. 
These approximate solutions are of the functional form necessitated by Theorems \ref{main lemma VICReg}, \ref{main lemma SCLNorm},\ref{main lemma Barlow Twins} and therefore can be used in Algorithm \ref{algorithm}.

\section{Experiments}

\paragraph{Visualizing the Augmentations.} We apply Algorithm \ref{algorithm} to MNIST \citep{lecun2010mnist} using representations obtained from a ResNet50 \citep{ResNet} architecture pretrained on ImageNet \citep{ImageNet}. We show the results for VICReg and SCL using different kernel functions in Figure \ref{fig:augmentations}. It can be observed that the augmented images are dissimilar from the original data, and that different kernel architectures result in semantically different augmentations. This is a direct consequence of the analytical form of the augmentations, which depend on the kernel matrix $K$.
 
 \paragraph{Evaluating the Reconstruction of $f^*$.} To verify our theoretical results, we measure the similarity between the target function $f^*$ and a minimizer $f$ of $\mathcal{L}$, under the augmentation $\mathcal{T_H}$. The proof of Theorems \ref{main lemma VICReg} and \ref{main lemma SCLNorm} shows that the affine transformation between $f$ and $f^*$ takes the form $f=WQf^\ast + b$, where $W$ is a whitening matrix, $b$ is the mean of the representations and $Q$ is an orthogonal matrix. Thus, in the case where $\text{cov}([f^*(x_1),\ldots  ,f^*(x_n)])=I_d$ and the representations are centered, we expect $f=Qf^*$. We verify this in our next set of experiments by computing the average Procrustes distance between $\{f^*(x_i)\}_{i=1}^n$ and $\{f(x_i)\}_{i=1}^n$. It is defined as
$\min_{Q\in \mathbb{R}^d,Q^\top Q=I_d}\frac{1}{n}\Vert F-QF^*\Vert_F$
where $\Vert\cdot \Vert_F$ is the Frobenius norm, $F^* = [f^*(x_1),\ldots ,f^*(x_n)]$ and $F=[f(x_1),\ldots ,f(x_n)]$.

\begin{figure}[!t]
  \centering
  \begin{subfigure}[b]{0.27\textwidth}
    \centering
    \includegraphics[width=\linewidth]{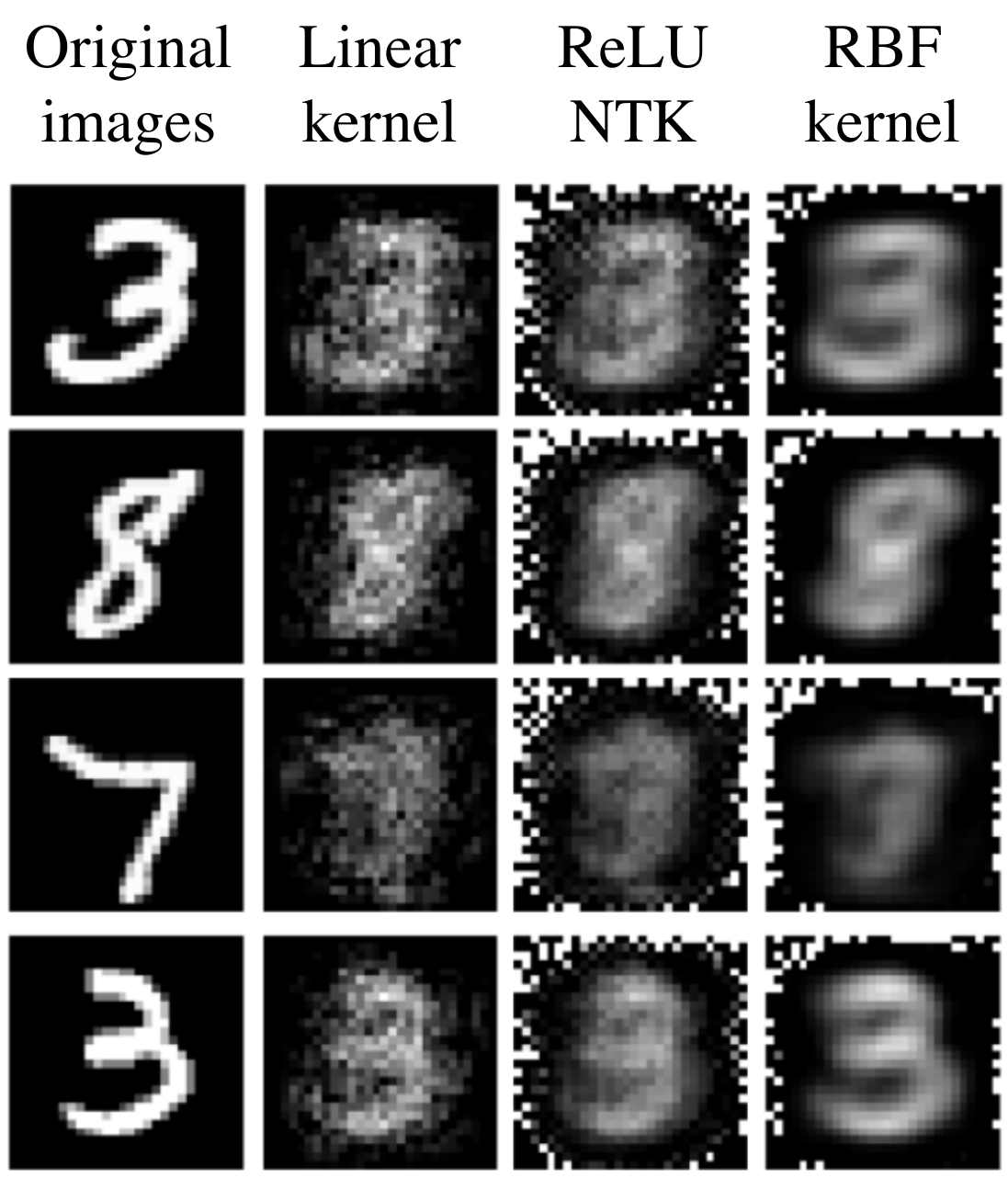}
    \subcaption{Input-space augmentations}\label{fig:augmentations}
  \end{subfigure}
  \hfill
  \begin{subfigure}[b]{0.31\textwidth}
    \centering
    \includegraphics[width=\linewidth]{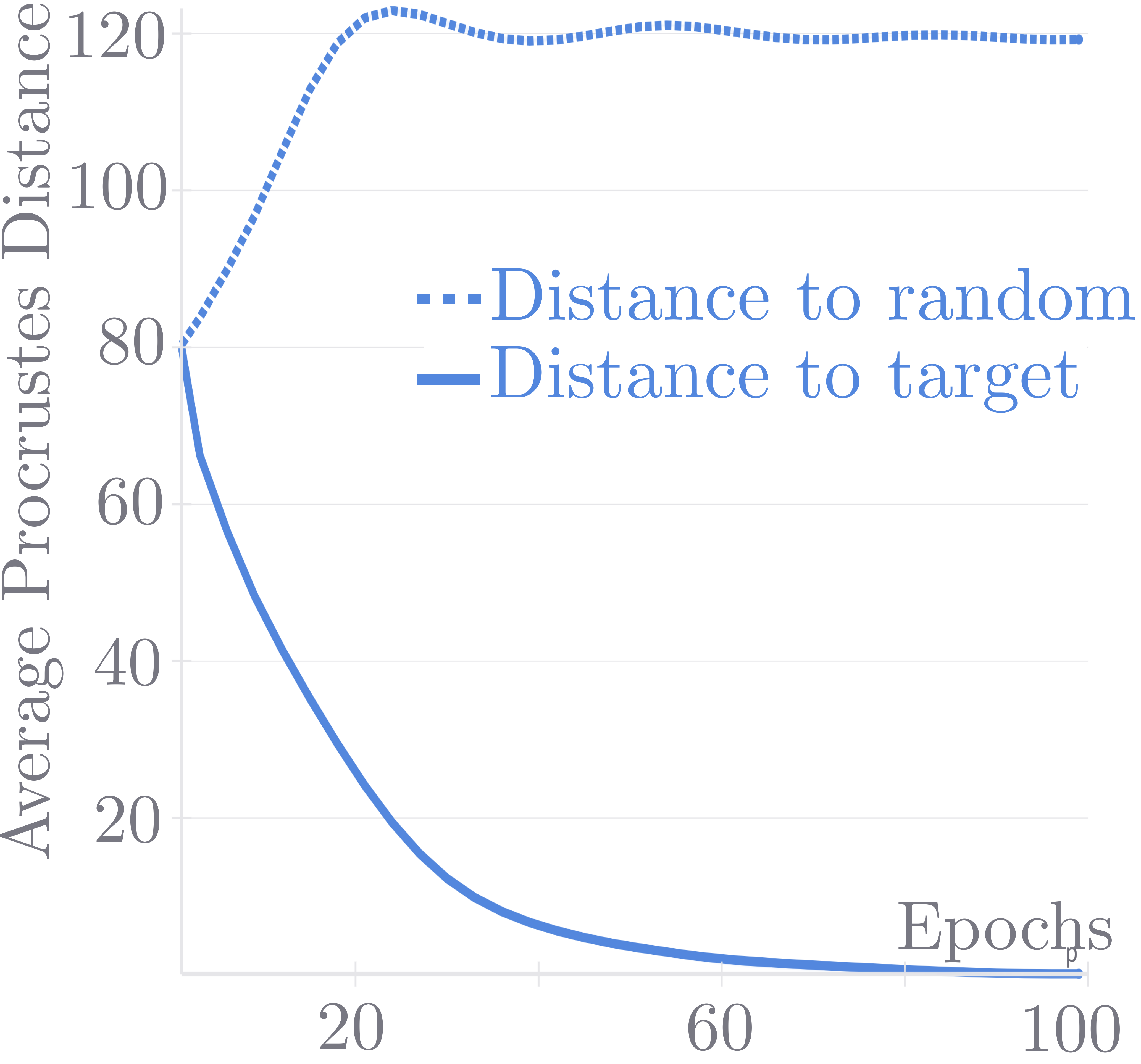}
    \subcaption{Spectral Constrastive Loss}\label{fig:reconstruction_scl}
  \end{subfigure}
  \hfill
  \begin{subfigure}[b]{0.32\textwidth}
        \centering
        \includegraphics[width=     \linewidth]{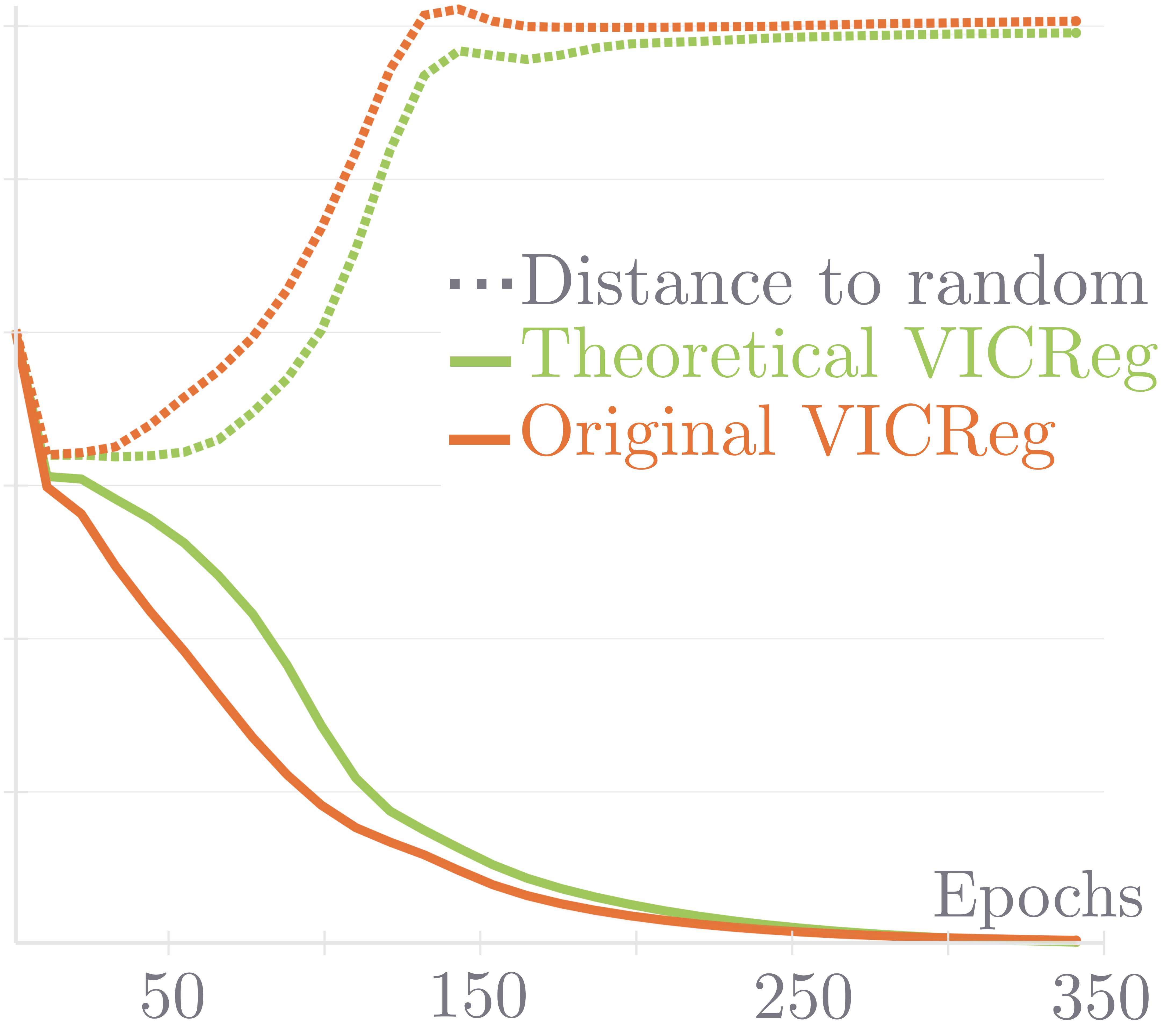}
        \subcaption{VICReg}\label{fig:reconstruction_real_vicreg}
      \end{subfigure}
      \hfill
  \caption{
    (a) Comparison of original and augmented MNIST images for different kernels. Notably, different function classes require different augmentations to achieve the same representations. 
     (b)-(c) The average Procrustes distance between the learned representations and target/random representations during training given the augmentations $\mathcal{T}_\mathcal{H}$. We consider SCL, VICReg as defined in Section~\ref{preliminaries} as well as the original VICReg; we achieve the target representations up to equivalence for all losses.
  }
  \label{fig:combined}
\end{figure}

 \paragraph{SCL, Theoretical VICReg, and VICReg.} In Figures \ref{fig:reconstruction_scl}-\ref{fig:reconstruction_real_vicreg}, We optimize $\mathcal{L}_\text{SCL}$, $\mathcal{L}_\text{VIC}$ and the original VICReg loss \citep{bardes_vicreg_2022} under the augmentation $\mathcal{T}_\mathcal{H}$ with target representations obtained from a ResNet50 pretrained on ImageNet. We compute the target representations for 10,000 MNIST images \citep{lecun2010mnist} and use the RBF kernel. We compare the Procrustes distance of the learned representations to the target representations, and to random representations with the same covariance. We observe that the representations learned with the constructed augmentation indeed achieve the target representations. It may seem surprising that the theoretical VICReg corresponds so closely to the original version, we explain this fact in the appendix. The final average Procrustes distances are 0.6 ± 0.1 for VICReg, 0.4 ± 0.1 for $\mathcal{L}_\text{VIC}$ and 0.004 ± 0.001 for $\mathcal{L}_\text{SCL}$.

\paragraph{Additional Experiments.} We repeat the experiments for a ViT-B/16 \citep{dosovitskiy2021imageworth16x16words} using CIFAR-10 \citep{Krizhevsky2009LearningML} and Tiny ImageNet \citep{Le2015TinyIV}. We optimize VICReg and measure the average procrustes distance after 300 epochs. The random baseline distance is 120.
\begin{table}[!htb]
  \centering
  \begin{tabular}{lccc}
    \toprule
    Model & MNIST & CIFAR-10 & Tiny ImageNet \\
    \midrule
    ResNet50 & 0.3 ± 0.2 & 0.14 ± 0.04 & 0.2 ± 0.2 \\
    ViT-B/16    & 0.7 ± 0.1 & 0.7 ± 0.2 & 0.4 ± 0.2 \\
    \bottomrule
  \end{tabular}
\end{table}

\section{Discussion}\label{discussion_sec}

Our work is the first to “invert” the typical mathematical analysis of SSL, asking what the optimal augmentation for a desired target representation must be. We answer this question for two of the most popular non-contrastive losses (VICReg and Barlow Twins) and the Spectral Contrastive Loss --- a theoretical proxy to SimCLR. Below, we discuss the main implications of our analysis.

\paragraph{The Role of Augmentations.}
Augmentations are commonly thought of as providing different views of the same data. Our theoretical analysis provides a more nuanced look at this matter. Theorems \ref{main lemma VICReg} and \ref{main lemma SCLNorm} allow to interpret augmentations as projections in the feature space to a subspace, which guarantees the minimizer $f$ of the SSL loss to be equivalent a specific $f^*$, meaning that \textit{augmentations can act as projections in the feature space instead of generators of different views.} Moreover, our results show that even one non-trivial augmentation is sufficient to learn good representations, and the distribution induced by this augmentation need not align with the original data distribution.  This provides a first theoretical justification for recent empirical observations \citep{moutakanni2024dontneeddataaugmentationselfsupervised}.

\paragraph{The Impact of Architecture.} 
Figure \ref{fig:augmentations} shows that for different architectures, the augmentations are recognizably different. This follows directly from our theory since the matrix $K$ differs depending on the learnable function class in our setup. 
 \citet{saunshi_understanding_2022} note that architecture has an implicit bias on the representations, and that this affects downstream performance. Our work also sheds light on the architectural influence in SSL, but from a new perspective: \textit{For the same target representation, the function class heavily influences the optimal augmentations needed to learn these representations.} 

\paragraph{The choice of $f^*$.}
Assuming a target $f^*$ is  standard in most theoretical analyses of supervised learning, and also common in SSL theory \citep{pmlr-v202-simon23a, zhai2024understanding}. By assuming an arbitrary target, we show that our results hold for any $f^*$.  While we choose $f^*$ to be a pretrained model in our experiments, there is no general consensus on what constitutes a good representation. 
When there is a natural target, our theoretical framework allows investigating the interplay between augmentations and learned representations from a new angle . As a demonstrating example, suppose the data-generating process is governed by a statistical model, such as the spiked covariance model commonly studied in high-dimensional statistics \citep{wainwright2019high}: $
    x \sim \mathcal{N}_m(0,\Sigma);
    \Sigma = \nu \cdot \theta \theta^\top + I_m$ 
 where $\nu > 0$ and $\theta \in \mathbb{R}^m$ is the (unknown) signal direction hidden in the covariance $\Sigma$. Previous works on SSL would only be able to answer what representations $\hat \theta^\top x$ are learned from a given augmentation. In contrast, our theoretical results from Section \ref{main_sec} would give us analytical expressions for what the ``correct'' augmentation is, i.e. the one that projects onto the signal direction $\theta$.

\paragraph{Connection Between VICReg and SCL.} Our results show that both VICReg and SCL can achieve the same $f^*$ given the same augmentations. The connection between VICReg and SCL has been shown before for the limit $n\rightarrow\infty$ \citep{pmlr-v202-cabannes23a}, and in the finite sample case, through replacing $c(Z)$ from the definition of VICReg with contrastive terms while keeping  the non-contrastive variance term $v(Z)$ \citep{garrido_duality_2022}. To the best of our knowledge, we are the first to show a theoretical connection between VICReg and SCL for finite data by regularizing \textit{only} with contrastive terms, thereby providing strong support to the notion of duality between contrastive and non-contrastive learning \citep{garrido_duality_2022}. Compared to previous works, we do not attempt to achieve functional equivalence between the losses, but instead show that the same augmentations can lead to the same $f^*$ when minimizing both losses, even if the optimization landscape is different.

\paragraph{Extension to Neural Networks.} The most significant limitation of our results is that, although neural networks can serve as the target $f^*$, the function class is assumed to be a kernel model, similar to previous theoretical studies on SSL. Our framework can extend, however, to neural networks as follows: Consider a parameterized function class $\Theta\phi_\theta(\cdot)$ and a target $f^*(x)=\Theta^*\phi_{\theta^*}(x)$. Taking the kernel $k_{\theta^*}(x,x')=\phi_{\theta^*}(x)^\top \phi_{\theta^*}(x')$, our theory provides augmentations such that $f^*$ uniquely optimizes the ``frozen'' neural network $\Theta\phi_{\theta^*}(\cdot)$, up to an affine transformation. Interestingly, $f^*$ will also be a global minimum over the ``unfrozen'' neural network $\Theta\phi_{\theta}(\cdot)$, as the proofs of Theorems~\ref{main lemma VICReg}, \ref{main lemma SCLNorm}, \ref{main lemma Barlow Twins} show that $f^*$ achieves the minimum possible value of the loss. The uniqueness of the minimum can only possibly be guaranteed in the infinite width limit, which is usually the case for neural networks, since otherwise $\phi_{\theta}$ changes during training. Nevertheless, it is encouraging that $f^*$ is a global minimizer also for finite-width neural networks, which motivates future research on the conditions under which $f^*$ or an equivalent function is achieved instead of other minima.

\begin{ack}
This work started as part of Shlomo Libo Feigin's Master's thesis at the Technical University of Munich. It was partially supported by the German Research Foundation (DFG) through the Research Grant GH 257/4-1, and by the DAAD programme Konrad Zuse Schools of Excellence in Artificial Intelligence, sponsored by the Federal Ministry of Education and Research (BMBF). We thank Vladimir Cherkassky, Daniel He, Daniel Bin Schmid, Nimrod de la Vega and Maedeh Zarvandi for the interesting and helpful discussions on this work. We also wish to thank Nil Ayday, Alexandru Craciun and Pascal Esser for the useful comments on the paper.
\end{ack}

\bibliographystyle{plainnat}
\bibliography{bib}

\appendix

\section{Definition of the VICReg Loss}\label{appendix_def_vicreg}
Although in Section 3 we use the variance in the definition of $v$ as part of the VICReg loss, we can use a variety of other definitions for our results. For example, we can use the standard deviation similar to the original VICReg definition \citep{bardes_vicreg_2022}. The part of the proof where the definition of VICReg comes to play is Lemma \ref{vicreg_insight}, where we use the fact that:
\[
v(Z)=0\iff \forall_{i\in [d]} \text{cov}(Z)_{i,i} =1
\]
Which is, of course, also true for:
\[
v(Z) = \frac{1}{d}\sum_{i=1}^{d}\left(1-\sqrt{[\text{cov}\left(Z\right)]_{i,i}}\right)^{2}  
\]
Similar to the definition in \citet{bardes_vicreg_2022}, except we use the $L_2$ loss instead of the hinge loss. ,In fact, any loss function $l(\cdot)$ where $l(y) = 0 \iff y=0$ would work for our results.

In addition, the unnormalized cross-correlation can be used instead of the covariance (Theorem \ref{mainLemmaVICRegCorr}). 

\section{Proof of Statements in Section 3.2}

\subsection{$\mathcal{L}_\text{SCL}$ Can Diverge Without Regularization}
The basic intuition behind this argument is that if we manage align the negative examples such that they have a close to $0$ inner-product, then the norm can grow and the contibution of the term $-\frac{2}n{}\sum_{i=1}^nz_i^\top z'_i$ will outweigh the contribution of the term $\frac{1}{n^2}\sum_{i\neq j}(z_i^\top z'_j)^2$.

More formally, to demonstrate that $\mathcal{L}_\text{SCL}$ can diverge to $-\infty$ without any regularization, we look at the following setup:

Let $\{x_i\}_{i=1}^n\subseteq \mathcal{X}$ be the data divided into $b=\frac{n}{d}\in\mathbb{N}$ batches $B_j=\{x_i\}_{i=j\cdot n+1}^{(j+1)\cdot n }$, where $d$ is the dimension of the representation and $\mathcal{X}$ is assumed to be a compact metric space. Let $\mathcal{T}$ be a distribution of augmentations $T:\mathcal{X}\rightarrow\mathcal{X}$ with a finite support and let $\mathcal{F}$ be a universal function class, i.e. for every continuous function $g:\mathcal{X}\rightarrow\mathbb{R}^d$ in $\mathcal{X}\rightarrow\mathbb{R}^d$ and every $\epsilon >0$, there exists $f\in \mathcal{F}$ such that $\sup_{x\in\mathcal{X}}\vert f(x)-g(x) \vert _2<\epsilon$. Assume the augmentations do not overlap, i.e. $\forall_{T,T'\in \mathcal{T}}\forall_{i\neq j}T(x_i)\neq T'(x_j)$. 

We define the unregularized loss as:

\begin{align*}
\mathcal{L}_\text{SCL-unreg}(f)&=\sum_{j=1}^b\mathbb{E}_{Z_j,Z_j'}[L_{\text{SCL-unreg}}(Z_j,Z_j)]\quad \text{, where for a batch size $\nu$:}\\
L_{\text{SCL-unreg}}(Z,Z')&=-\frac{2}{\nu}\sum_{i=1}^{\nu}z_i^\top z'_i+\frac{1}{\nu^2}\sum_{i\neq j}(z_i^\top z'_j)^2
\end{align*}

 and $Z_j$ corresponds to the $j$'s batch. Namely, for $\left(T_i, T'_i\right) {\sim} \mathcal{T}\otimes\mathcal{T}$:
\begin{align*}
  Z_j&=[f(T_{j\cdot d+1}(x_{j\cdot d+1})),\ldots, f(T_{(j+1)\cdot d}(x_{(j+1)\cdot d}))]  \\
  Z_j'&=[f(T_{j\cdot d+1}'(x_{j\cdot d+1})),\ldots, f(T_{(j+1)\cdot d}'(x_{(j+1)\cdot d}))]
\end{align*}

\begin{claim*}
   In the setting above, $\inf_{f\in\mathcal{F}}\mathcal{L}_\text{SCL-unreg}(f)=-\infty$ 
\end{claim*}
\begin{proof}

For a $\lambda>0$ Let $g(x)$ be a function such that $\forall_{T\in \text{Support}(\mathcal{T})} g(T(x_i))=\sqrt\frac{\lambda}{2b}\cdot e_{i\bmod d}$, where $e_k$ is the $k$'th standard basis vector. $g_\lambda(\cdot)$ is well defined on the finite set $S=\{T(x_i):i\in [n], T\in \text{Support}(\mathcal{T})\}$ since we assumed  the augmentations do not overlap, i.e. $\forall_{T,T'\in \mathcal{T}}\forall_{i\neq j}T(x_i)\neq T'(x_j)$. The function $g_\lambda$ can be extended from the space $S$ to $\mathcal{X}$, such that $g_\lambda$ is continuous, via the Tietze extension theorem. Therefore, the universality of $\mathcal{F}$ implies that for every $\epsilon >0$ there exists an $f\in \mathcal{F}$ such that $\sup_{x\in\mathcal{X}}\vert f(x)-g(x)\vert _2 <\epsilon $

Now, we can prove $\mathcal{L}_\text{SCL-unreg}(g_\lambda)=-\lambda$:
\begin{align*}
  Z_j&=[g_\lambda(T_{j\cdot d+1}(x_{j\cdot d+1})),\ldots, g_\lambda(T_{(j+1)\cdot d}(x_{(j+1)\cdot d}))]  \\
&=[\sqrt\frac{\lambda}{2b}\cdot e_1,\ldots \sqrt\frac{\lambda}{2b}\cdot e_d] =\sqrt\frac{\lambda}{2b}I_d
\end{align*}

and similarly $Z_j'=\sqrt\frac{\lambda}{2b}I_d$. Therefore, $L_{\text{SCL-unreg}}(Z,Z')=-\frac{2}{d}\sum_{i=1}^d e_i ^\top e_i+\frac{1}{d^2}\sum_{i\neq j}e_i^ \top e_j =-\frac{\lambda}{b}$ and $\mathcal{L}_\text{SCL-unreg}(g_{\lambda})=\sum_{j=1}^b\mathbb{E}_{Z_j,Z_j'}[L_{\text{SCL-unreg}}(Z_j,Z_j)]=-b\cdot \frac{\lambda}{b}=-\lambda$.

$\mathcal{L}_{\text{SCL-unreg}}(f)$ can be viewed as a continuous function from the set $\{f(T(x_i)):T\in \text{Support}(\mathcal{T}),i\in [n]\}$ and therefore the universality of $\mathcal{F}$ implies that for every $\delta>0$ there is a $f\in\mathcal{F}$ such that $\mathcal{L}_{\text{SCL-unreg}}(f_\lambda)\leq \mathcal{L}_{\text{SCL-unreg}}(g_{\lambda})+\delta$, in particular, we can fix $\delta=1$ and get that for every $\lambda>0$ there exists $f_\lambda\in\mathcal{F}$ such that $\mathcal{L}_{\text{SCL-unreg}}(f_\lambda)\leq -\lambda+1$. This allows to prove our claim that $\inf_{f\in\mathcal{F}}\mathcal{L}_\text{SCL-unreg}(f)=-\infty$ by the following standard argument: Assume by contradiction $\inf_{f\in\mathcal{F}}\mathcal{L}_\text{SCL-unreg}(f)=M$ where $M\in \mathbb{R}$. For $\lambda=M-2$ we get $\exists_{f\in\mathcal{F}}\mathcal{L}_\text{SCL-unreg}(f)\leq M-1$ which contradicts the assumption $\inf_{f\in\mathcal{F}}\mathcal{L}_\text{SCL-unreg}(f)=M$ and therefore proving $\inf_{f\in\mathcal{F}}\mathcal{L}_\text{SCL-unreg}(f)=-\infty$.
\end{proof}

\subsection{Lower Bound on $\mathcal{L}_\text{SCL}$}\label{lower_bound}

We begin by defining:
\[
L_{\text{SCL}}(Z,Z')\coloneqq\Vert Z^\top Z'-\text{diag}(Z^\top Z')\Vert_F^2 - \frac{2}{n}\text{Tr}(Z^\top Z')+\frac{1}{2n^2}\Big(\sum_{i=1}^n \Vert z_i\Vert^4 + \Vert z_i'\Vert^4 \Big)
\]

It can easily be verified that $\mathcal{L_{\text{SCL}}}=\mathbb{E}_{Z,Z'}\big[L_{\text{SCL}}(Z,Z')\big]$. We will show that $\mathcal{L}_{\text{SCL}}\geq -d$ by showing that $L_{\text{SCL}}(Z,Z')\geq -d$. 

\begin{align*}
    &L_{\text{SCL}}(Z,Z') =\frac{1}{n^2}\Vert Z^\top Z'-\text{diag}(Z^\top Z')\Vert_F^2 - \frac{2}{n}\text{Tr}(Z^\top Z')+\frac{1}{2n^2}\Big(\sum_{i=1}^n \Vert z_i\Vert^4 + \Vert z_i'\Vert^4 \Big)\\
    &=\frac{1}{n^2}\Vert Z^\top Z'\Vert_F^2-\frac{2}{n^2}\underbrace{\text{Tr}\big((Z^\top Z')^\top \text{diag}(Z^\top Z')\big)}_{\Vert\text{diag}(Z^\top Z')\Vert_F^2} +\frac{1}{n^2}\Vert\text{diag}(Z^\top Z')\Vert_F^2 \\& -\frac{2}{n}\text{Tr}(Z^\top Z')+\frac{1}{2n^2}\Big(\sum_{i=1}^n \Vert z_i\Vert^4 + \Vert z_i'\Vert^4 \Big)\\
    &=\frac{1}{n^2}\Vert Z^\top Z'\Vert_F^2 - \frac{2}{n}\text{Tr}(Z^\top Z')+\frac{1}{2n^2}\Big(\sum_{i=1}^n \Vert z_i\Vert^4 + \Vert z_i'\Vert^4 \Big)-\frac{1}{n^2}\Vert\text{diag}(Z^\top Z')\Vert_F^2\\
    &=\frac{1}{n^2}\Vert Z^\top Z'\Vert_F^2 - \frac{2}{n}\text{Tr}(Z^\top Z')+\frac{1}{2n^2}\Big(\sum_{i=1}^n \Vert z_i\Vert^4 -2\Vert z_i^\top  z'_i\Vert^2+ \Vert z_i'\Vert^4 \Big)\\
    &=\frac{1}{n^2}\Vert Z^\top Z'\Vert_F^2 - \frac{2}{n}\text{Tr}(Z^\top Z')+\frac{1}{2n^2}\Big(\sum_{i=1}^n \Vert z_i z_i^\top -z_i'z_i'^\top \Vert^2_F \Big)\\
    &\geq \frac{1}{n^2}\Vert Z^\top Z'\Vert_F^2 - \frac{2}{n}\text{Tr}(Z^\top Z')+\frac{1}{2n^2}\Big(\Vert \sum_{i=1}^n z_i z_i^\top -z_i'z_i'^\top \Vert^2_F \Big)\quad \numberthis\label{inequality}\\
    &=\frac{1}{n^2}\Vert Z^\top Z'\Vert_F^2 - \frac{2}{n}\text{Tr}(Z^\top Z')+\frac{1}{2}\Big(\Vert \frac{1}{n}ZZ^\top -\frac{1}{n}Z'Z'^\top  \Vert^2_F \Big)\\
    &=\frac{1}{n^2}\Vert Z^\top Z'\Vert_F^2 - \frac{2}{n}\text{Tr}(Z^\top Z')+\frac{1}{2}\Big(\Vert (\frac{1}{n}ZZ^\top  - I_d) - (\frac{1}{n}Z'Z'^\top -I_d) \Vert^2_F \Big)\\
    &=\frac{1}{n^2}\Vert Z^\top Z'\Vert_F^2 - \frac{2}{n}\text{Tr}(Z^\top Z')+\frac{1}{2}\Vert \frac{1}{n}ZZ^\top  - I\Vert_F^2  + \frac{1}{2}\Vert \frac{1}{n}Z'Z'^\top  - I_d\Vert_F^2  \\& -\text{Tr}\Big((\frac{1}{n}ZZ^\top -I_d)^\top (\frac{1}{n}Z'Z'^\top -I_d) \Big)
\end{align*}

\eqref{inequality} is a consequence of the triangle inequality.

\begin{align*}
&\text{Tr}\Big((\frac{1}{n}ZZ^\top -I)^\top (\frac{1}{n}Z'Z'^\top -I) \Big)\\&=\frac{1}{n^2}\underbrace{\text{Tr}(ZZ^\top Z'Z'^\top )}_{=\text{Tr}(Z'^\top ZZ^\top Z')} -\frac{1}{n}\text{Tr}(ZZ^\top )-\frac{1}{n}\text{Tr}(Z'Z'^\top )+\text{Tr}(I_d)\\
&=\frac{1}{n^2}\text{Tr}\big((Z^\top Z')^\top (Z^\top Z')\big)-\frac{1}{n}\text{Tr}(ZZ^\top )-\frac{1}{n}\text{Tr}(Z'Z'^\top )+\text{Tr}(I_d)\\
&=\frac{1}{n^2}\Vert Z^\top Z'\Vert_F^2 -\frac{1}{n}\text{Tr}(ZZ^\top )-\frac{1}{n}\text{Tr}(Z'Z'^\top )+d\\
&=\frac{1}{n^2}\Vert Z^\top Z'\Vert_F^2 -\frac{1}{n}\big(\text{Tr}(ZZ^\top )-2\text{Tr}(Z^\top Z') +  \text{Tr}(Z'Z'^\top )\big) +\frac{2}{n}\text{Tr}(Z^\top Z') +d\\
&=\frac{1}{n^2}\Vert Z^\top Z'\Vert_F^2 -\frac{1}{n}\Vert Z-Z' \Vert_F^2 +\frac{2}{n}\text{Tr}(Z^\top Z') +d\\
\end{align*}

Combining everything together, we get:
\begin{align}
    L_{\text{SCL}}(Z,Z')\geq\frac{1}{2}\Vert \frac{1}{n}ZZ^\top  - I\Vert_F^2  + \frac{1}{2}\Vert \frac{1}{n}Z'Z'^\top  - I_d\Vert_F^2+\frac{1}{n}\Vert Z-Z'\Vert_F^2 -d \geq -d\label{second_ineq} 
\end{align}

The inequality is tight if and only if $\frac{1}{n}ZZ^\top =I_d$ and $Z=Z'$, because only then both \eqref{inequality}  and the right side of \eqref{second_ineq} are equalities.

\section{Proofs of the Main Results}\label{appendix:proofs}

\paragraph{Additional Notation and Definitions.} In addition to the notation in Section 3, we denote $\phi_i\coloneqq\phi(x_i)$, $\Vert\cdot\Vert_\text{HS}$ to be the Hilbert-Schmidt norm, and $O(n)$ to be the set of orthogonal $n\times n$ matrices. Moreover, we define $L_{\text{VIC}}, L_{\text{BT}},L_{\text{SCL}}$ such that the loss functions in Section 3.2 are of the form $\mathcal{L}=\mathbb{E}_{Z,Z'}\big[L(Z,Z')\big]$, namely:
\begin{align*}
L_{\text{VIC}}(Z,Z')&\coloneqq \lambda s(Z, Z') + \mu [ v(Z) + v(Z')] + \nu [ c(Z) + c(Z')]   \\
L_{\text{BT}}(Z,Z')&\coloneqq \sum_i^n (1- \mathcal{C}_{ii})^2+ \lambda\sum_{i\neq j} (1- \mathcal{C}_{ij})^2 ;\quad \mathcal{C}=\frac{1}{2n}(ZZ'^\top  + Z'Z^\top )\\
L_{\text{SCL}}(Z,Z')&\coloneqq\Vert Z^\top Z'-\text{diag}(Z^\top Z')\Vert_F^2 - \frac{2}{n}\text{Tr}(Z^\top Z')+\frac{1}{2n^2}\Big(\sum_{i=1}^n \Vert z_i\Vert^4 + \Vert z_i'\Vert^4 \Big)
\end{align*}

\subsection{Proof of Theorem \ref{main lemma VICReg} (VICReg)}
\label{vic_appendix}
 To prove Theorem \ref{main lemma VICReg} we need a number of auxiliary results. The main part of the proof is in Proposition \ref{aux_vicreg}, where we prove our main result except we assume $\text{cov}\left([{f}^*(\phi_1), \ldots, {f}^*(\phi_n)]\right)=I_d$ instead of the covariance being full rank (Condition \ref{non_redundant_representation}), which is easily relaxed afterwards. 
 
 First, we derive a representer-like theorem that guarantees \textbf{any} $f$ that minimizes VICReg with least norm to be of the form $f = A \Phi^\top $, provided a suitable augmentation.
 
\begin{restatable}{proposition}{representer}\label{representer}
    Let $\mathcal{L}$ be a joint embedding loss as described in Section 3. Let $\mathcal{T}$ be a distribution of augmentations such that $\forall_{T\in\text{Support}(\mathcal{T})}T\Phi \in \{\Phi A K:A\in\mathbb{R}^{n\times n}\}$. Suppose $ f \in\arg\min\{\Vert f \Vert_{HS}:\mathcal{L}(\{\phi_i\}_{i=1}^{n}, \mathcal{T}, f)=0\}$. Then, $f=A\Phi^\top $ for some $A\in\mathbb{R}^{d\times n}$.
\end{restatable}
\begin{proof}
	Let $\mathcal{T}$ be a distribution of augmentions s.t. $\forall_{T\in\text{Support}(\mathcal{T})}T\Phi \in \{\Phi A K:A\in\mathbb{R}^{n\times n}\}$ and let $f\in \mathbb{R}^d\otimes\mathcal{H}$ be a solution to $\mathcal{L}(\{\phi_i\}_{i=1}^{n}, \mathcal{T}, f)=0$ . Define $S_\parallel \coloneqq\{A\Phi^\top : A\in \mathbb{R}^{d\times n }\}$. We decompose $f=f_\parallel+f_\perp$ where $f_\parallel\in S_\parallel$ and $<f_\perp, \hat{f}>_{\text{HS}}=0$ for every $\hat{f}\in S_\parallel$. Since $S_\perp$ is finite-dimensional, such a decomposition exists via the Gram-Schmidt process. 
	
	$\mathcal{L}$ is a function of $\{f(T(\phi_i): T\in\text{Support}(\mathcal{T}), i\in [n]\}$, and since  we assume $\forall_{T\in\text{Support}(\mathcal{T})}T\Phi \in \{\Phi A K:A\in\mathbb{R}^{n\times n}\}$, we get that $\mathcal{L}$ is a function of $\{f(\phi_i)\}_{i=1}^n$. 
	
	We first prove that $\forall_{i\in [n]}f_{\perp}(\phi_i)=0$. Let $A\in\mathbb{R}^{d\times n}$ be an arbitrary matrix. 
	\begin{gather*}
		0=<f_\perp, A\Phi^\top >_{\text{HS}} = \text{Tr}(f_\perp^\top  A\Phi^\top ) = \text{Tr}(A\Phi^\top  f_\perp^\top )= \sum_{i=1}^n <Ae_i, \Phi^\top  f_\perp^\top  e_i> 
	\end{gather*}
	Where $\Phi^\top  f_\perp^\top $ is a $n\times n$ matrix.  Since $A$ is arbitrary it means that $\Phi^\top  f_\perp^\top  = 0$ and hence $(\Phi^\top  f_\perp^\top )^\top =f_\perp\Phi=0$ . Therefore: $\forall_{i\in [n]}f_{\perp}(\phi_i)=0$. We get that for every $i\in [n]$:
	\begin{align*}
		f(\phi_i)&=f_\parallel(\phi_i) \\
		\Vert f \Vert_{\text{HS}}^2  &= \Vert f_\parallel \Vert_{\text{HS}}^2 +\Vert f_\perp \Vert_{\text{HS}}^2
	\end{align*}
	And since $f$ is a min-norm solution of $\mathcal{L}(\{\phi_i\}_{i=1}^{n}, \mathcal{T}, f)=0$ it must satisfy $f=f_\parallel\in S_\parallel$.
	
\end{proof}

We continue by deriving two conditions that are equivalent to obtaining zero loss in VICReg.

\begin{restatable}{lemma}{vicregInsight}\label{vicreg_insight}
 $L_{\text{VIC}}(Z,Z')=0$ if and only if $Z=Z'$ and $\text{cov}(Z)=n^{-1}ZHH^\top Z^\top =I$.
\end{restatable}
\begin{proof}
	$L_{\text{VIC}}(Z,Z')=0$ if and only if:
	\begin{align*}
		s(Z,Z') &= \frac{1}{n}\sum_{i=1}^{n}\left\Vert z_{i}-z_{i}'\right\Vert _{2}^{2} = 0\\
		v(Z) &= \frac{1}{d}\sum_{i=1}^{d}\left(1-[\text{cov}\left(Z\right)]_{i,i}\right)^{2} = 0 \\
		c(Z) &= \frac{1}{d}\sum_{i\neq j}^{d}[\text{cov}(Z)]_{i,j}^2 = 0
	\end{align*}
	
	Which is satisfied if and only if $Z=Z'$ and $\text{cov}(Z) = I$.
	
\end{proof}

Next, we show a statement that is already very close to Theorem \ref{main lemma VICReg}. The only difference here is that we still assume the target representations to have identity covariance. This simplifies the proof, but it will be easy to relax.

\begin{proposition}\label{aux_vicreg}
  Suppose Condition \ref{invertable_gram} is satisfied. Let $f^*\in\mathbb{R}^d\otimes\mathcal{H}$ be of the form $f^*=C\Phi^\top,C \in \mathbb{R}^{d\times n}$ s.t. $\text{cov}\left([f^*(\phi_1), \ldots, f^*(\phi_n)]\right)=`I_d$. Then, if $f$ minimizes $\mathcal{L}_{\text{VIC}}\left(\left\{ \phi_{i}\right\} _{i=1}^{n}, \mathcal{T}_\mathcal{H}(C),f\right)$ over $\mathbb{R}^d\otimes\mathcal{H}$ with least Hilbert-Schmidt-norm, there exists $Q\in O(d)$ s.t. $f=Qf^*$.

\end{proposition}
 \begin{proof}  
 Our goal in this proof is to show that the set of least norm minimizers of $\mathcal{L}_{\text{VIC}}$, which will be denoted by $\mathcal{F^*}$, is the same as $\{Qf^*: Q\in O(d)\}$. The general structure of the proof is as follows: We first prove that any $Qf^*$ indeed achieves minimum loss (Claim \ref{first_direction}) and then prove that any least norm minimizer must be of this form (Claim \ref{claim_2}) --- which is the crux of the proof. Claims \ref{first_direction} and \ref{claim_2}
will not be sufficient by themselves to prove the equivalence between $\{Qf^*: Q\in O(d)\}$ and $\mathcal{F^*}$, since it could still be that there is a $Qf^*$ that minimizes the loss (which we know by Claim \ref{first_direction}) but is not a least norm solution. We prove that this cannot be the case in Claim \ref{claim_3} and thereby prove the proposition. 
 
 We first show that $\mathcal{T}_\mathcal{H}(C)$ is well defined by proving that $CKC^\top $ has full rank. This follows from the fact that if $K$ is a positive definite symmetric matrix (Condition \ref{invertable_gram}), and if $K=LL^\top $ is the Cholesky decomposition of $K$, $CKC^\top =(CL)(CL)^\top $ and $rank(CKC^\top )=rank(CL)=rank(C)$. The fact that $rank(C)=d$ follows our assumption that: $\text{cov}\left([f^*(\phi_1), \ldots, f^*(\phi_n)]\right)=\text{cov}(C\Phi^\top \Phi)=\frac{1}{n}(C\Phi^\top \Phi H)(C\Phi^\top \Phi H)^\top =I_d$ and because of the submultiplicativity of the matrix rank $rank(C)\geq rank(I_d)=d$. Of course $rank(C)\leq d$ because $C$ is a  $d\times n$ matrix. To summarize, we get $rank(CKC^\top) = rank(C)= d$.

For simplicity of notation we denote $\mathcal{T}=\mathcal{T}_\mathcal{H}(C)$. Define:
	\[
	\mathcal{F}^\ast\coloneqq\arg\min_{f\in\mathbb{R}^d\otimes\mathcal{H}}\{\Vert f\Vert_{\text{HS}}: \mathcal{L}_{\text{VIC}}\left(\left\{ \phi_{i}\right\} _{i=1}^{n},\mathcal{T},f\right)=0\}
	\]
	We will prove that $\mathcal{F}^*=\{QC\Phi^\top : Q\in O(d)\}$. Since $\mathcal{L}_{\text{VIC}}$ is non-negative, the property $\mathcal{L}_{\text{VIC}}=0$ also implies that  $\mathcal{F}^\ast$ is the set of minimizers of $\mathcal{L}_{\text{VIC}}\left(\left\{ \phi_{i}\right\} _{i=1}^{n},\mathcal{T},f\right)$ with least Hilbert-Schmidt-norm. 
	
	The fact that $\mathcal{L}_{\text{VIC}}$ is non-negative also implies that $\mathcal{L}_{\text{VIC}}(\{\phi_i\}_{i=1}^{n}, \mathcal{T}, f)=\mathbb{E}_{Z,Z'}\big[L_{\text{VIC}}(Z,Z')\big]=0$ if and only if every term inside the expectation is exactly $0$. i.e. for every $(T_1,T'_1,\ldots, T_n, T'_n)$ in the support of $(\mathcal{T})^{\otimes 2n}$
	\begin{gather}
		L_{\text{VIC}}([f(T_1(\phi_1)),\ldots, f(T_n(\phi_n))],[f(T'_1(\phi_1)),\ldots, f(T'_n(\phi_n))])=0\label{expectation_result}
	\end{gather}
	
	\begin{claim} \label{first_direction}
		We prove that $\{QC\Phi^\top : Q\in O(d)\} \subseteq \{f: \mathcal{L}_{\text{VIC}}\left(\left\{ \phi_{i}\right\} _{i=1}^{n},\mathcal{T},f\right)=0\}$.
	\end{claim}
	Let $f=QC\Phi^\top $ where $Q\in O(d)$. We prove that for every $T, T'\in \text{Support}(\mathcal{T})$ and for every $i\in[n]$ it holds that $f(T(\phi_i))=f(T'(\phi_i))$. Since $\text{Support}(\mathcal{T})=\{I_\mathcal{H}, \Phi C^{T}\left(CKC^{T}\right)^{-1}C\Phi^{T}\}$, it is enough to check that the following holds for every $i\in[n]$:
	\[
	QC\Phi^\top I_\mathcal{H}\phi_i= QC\Phi^\top  (\Phi C^{T}\left(CKC^{T}\right)^{-1}C\Phi^{T})\phi_i
	\]
	Which indeed holds since:
	\begin{gather*}
		QC\Phi^\top  \Phi C^{T}\left(CKC^{T}\right)^{-1}C\Phi^{T}\phi_i = Q\underbrace{CKC^{T}\left(CKC^{T}\right)^{-1}}_{I}C\Phi^{T}\phi_i  =QC\Phi^{T}\phi_i
	\end{gather*}
	Therefore, \eqref{expectation_result} is exactly:
	\begin{gather*}
		L_{\text{VIC}}(QC\Phi^{T}\Phi,QC\Phi^{T}\Phi)=0 \\
		L_{\text{VIC}}(QCK,QCK)=0 
	\end{gather*}
	Which according to Lemma \eqref{vicreg_insight} holds if $\text{cov}(QCK)=I$, which is true since we assumed that $\text{cov}(CK)=I$ and therefore indeed $L_{\text{VIC}}(QCK,QCK)=0$, \eqref{expectation_result} holds and Claim \ref{first_direction} is proven.
	
	\begin{claim}\label{claim_2}
		We prove that $\mathcal{F}^\ast\subseteq \{QC\Phi^\top : Q\in O(d)\}$.
	\end{claim}
	
	Let $f\in\mathcal{F}^\ast$. The structure of the proof will be as follows: First, we prove that $f$ is of the form $f=A\Phi^\top $ via the version of the representer theorem we proved (Proposition \ref{representer}). Then, we use Lemma \ref{vicreg_insight} which states the conditions under which $L(Z,Z')=0$ and substituting different possible augmentations combinations that could occur under $\mathcal{T}$ into $Z,Z'$ and get equations \eqref{first_equality} and \eqref{covariance_a} (we use the fact that no matter the combination, $L(Z,Z')=0$, which is formalized by eq. \ref{expectation_result}). Combining \eqref{first_direction}, \eqref{covariance_a} and the assumption that the covariance is identity will force $A$ to be of the form $A=QC$, which proves the claim.
    
    To start, to prove $f=A\Phi^\top $ we would like to use Proposition \ref{representer}, which requires $\forall_{T\in\text{Support}(\mathcal{T})}T\Phi \in \{\Phi \bar{A} K:\bar{A}\in\mathbb{R}^{n\times n}\}$. For $T=\Phi C^{T}\left(CKC^{T}\right)^{-1}C\Phi^{T}$ the condition is satisfied since by definition $K=\Phi^\top \Phi$. For $T=I_\mathcal{H}$ the condition is satisfied for $\bar{A}=K^{-1}$.
	Therefore, according to Proposition \ref{representer}  $f=A\Phi^\top $.

	We now use \eqref{expectation_result}. According to Lemma \ref{vicreg_insight} for every $i\in[n]$ and $T,T'\in \text{Support}(\mathcal{T})$:
	\begin{gather*}
		A\Phi^\top  T_i(\phi_i)=A\Phi^\top  T_i'(\phi_i) 
	\end{gather*}
	Taking $T_i=I_\mathcal{H}$ and $T_i'=\Phi C^{T}\left(CKC^{T}\right)^{-1}C\Phi^{T}$ results in the following:
	\begin{gather*}
		A\Phi^\top  I_\mathcal{H}(\phi_i)=A\Phi^\top  \Phi C^{T}\left(CKC^{T}\right)^{-1}C\Phi^{T}\phi_i 
	\end{gather*}
	Which holds for every $i\in[n]$, meaning:
	\begin{gather}
		AK=AK C^{T}\left(CKC^{T}\right)^{-1}CK \label{first_equality}
	\end{gather}
	We use Lemma \ref{vicreg_insight} again for \eqref{expectation_result} to obtain that for $T_1 = \ldots = T_n = T'_1 = \ldots = T'_n = I_\mathcal{H}$
	\begin{gather*}
		\text{cov}(A\Phi^\top  T_1(\phi_1),\ldots, A\Phi^\top  T_n(\phi_n))= \text{cov}(A\Phi^\top \Phi)= \text{cov}(AK) =I\\
		n^{-1}(AKH)(AKH)^\top =I\\
		n^{-1}AKHH^\top KA^\top =I \numberthis \label{covariance_a}
	\end{gather*}
	In addition, we assumed that $\text{cov}(CK)=I$, Therefore:
	\begin{gather*}
		n^{-1}CKHH^\top KC^\top =I \numberthis \label{covariance_c}
	\end{gather*}

We substitute \eqref{first_equality} into \eqref{covariance_a} and get:
\begin{gather*}
	n^{-1}AK C^{T}\left(CKC^{T}\right)^{-1}CKHH^\top KC^\top \left(CKC^{T}\right)^{-1}CKA^\top =I
\end{gather*}
We substitute \eqref{covariance_c} into the above equation and get:

\begin{gather*}
	AK C^{T}\left(CKC^{T}\right)^{-1}\left(CKC^{T}\right)^{-1}CKA^\top =I \numberthis \label{orthogonal}
\end{gather*}

We denote $Q=AK C^{T}\left(CKC^{T}\right)^{-1}$, the above equation means that $Q\in O(d)$. Substituting $Q$ back to \eqref{first_equality} we get:
\begin{gather*}
	AK=QCK \\
	A=QC
\end{gather*}
Therefore, $f\in \{QC\Phi^\top : Q\in O(d)\}$ and Claim \ref{claim_2} is proven.

\begin{claim}\label{claim_3}
	$\{QC\Phi^\top : Q\in O(d)\}\subseteq \mathcal{F}^\ast$.
\end{claim}

We assume by way of contradiction that there exists $f=\tilde{Q}C\Phi^\top \in \{QC\Phi^\top : Q\in O(d)\}$ such that $f\notin \mathcal{F}^*$. We know via Claim \ref{first_direction} that $\mathcal{L}_{\text{VIC}}\left(\left\{ \phi_{i}\right\} _{i=1}^{n},\mathcal{T},\tilde{f}\right)=0$, therefore the only way $f\notin \mathcal{F}^*$ is by not being a least-norm minimizer, i.e. there must be a $f'\in\mathcal{F}^\ast$ such that $\Vert\ f'\Vert_{\text{HS}}<\Vert f\Vert_{\text{HS}}$. Applying Claim \ref{claim_2} we know that $f'=\tilde{Q}'C\Phi^\top $ for $\tilde{Q}'\in O(d)$. We get that $\Vert\tilde{Q}C\Phi^\top \Vert_{\text{HS}}<\Vert\tilde{Q}'C\Phi^\top \Vert_{\text{HS}}$ for two matrices $\tilde{Q},\tilde{Q}'\in O(d)$. Which is a contradiction since for every $Q\in O(d)$:
\begin{gather*}
	\Vert QC\Phi^\top \Vert_{\text{HS}}^2=Tr((QC\Phi^\top )^\top (QC\Phi^\top ))=Tr((Q\Phi^\top )^\top Q^\top Q(C\Phi^\top ))\\ =Tr((C\Phi^\top )^\top (C\Phi^\top )) =  \Vert C\Phi^\top \Vert_{\text{HS}}^2 
\end{gather*}

Therefore, we have proven Claim \ref{claim_3}. Combining Claims \ref{claim_2} and \ref{claim_3} we get:
\[
\{QC\Phi^\top : Q\in O(d)\}= \mathcal{F}^\ast
\]
\end{proof}

With Proposition \ref{aux_vicreg} in hand, we prove Theorem \ref{main lemma VICReg}; this time, we do not assume a whitened covariance:

\mainLemmaVICReg*

\begin{proof}
The only difference between Theorem \ref{main lemma VICReg} and Proposition \ref{aux_vicreg} is that now instead of assuming  $f^*=C\Phi^\top,C \in \mathbb{R}^{d\times n}$ s.t. $\text{cov}\left([f^*(\phi_1), \ldots, f^*(\phi_n)]\right)=I_d$ we just assume it is full rank. However, since the covariance is full rank and symmetric, there is a whitening matrix $W$ s.t. $W\text{cov}\big([f^*(\phi_1), \ldots, f^*(\phi_n)]\big)W^\top =\text{cov}\big([Wf^*(\phi_1), \ldots, Wf^*(\phi_n)]=I$ and $W$ is invertible. From Proposition \ref{aux_vicreg} we get that if $f$ is a least norm minimizer of 
$\mathcal{L}_{\text{VIC}}\Bigl(\{\phi(x_i)\}_{i=1}^n,\;\mathcal{T}_\mathcal{H}(C),\;f\Bigr)$
over $\mathbb{R}^d \otimes \mathcal{H}$, it satisfies $f=QWf^*$. In particular, $f\overset{\text{aff}}{\sim}f^*$, where the affine transformation $QW$ is invertible as a product of two invertible matrices.
\end{proof}

\subsection{Proof of Theorem \ref{main lemma SCLNorm} (Spectral Contrastive Loss)}

The structure of our proof is as follows: First, we define a VICReg loss $\mathcal{L}_{\text{VIC-corr}}$ that uses the unnormalized cross-correlation instead of the covariance. We will verify that an equivalent version of Theorem \ref{main lemma VICReg} holds for $\mathcal{L}_{\text{VIC-corr}}$ with the same augmentations $\mathcal{T_H}$ as in Definition \ref{definition vicreg}. Then, we will use the inequality derived in Section \ref{lower_bound} to show that $L(Z,Z')_{\text{VIC-corr}}=0\iff {L}_{\text{SCL}}(Z,Z')=-d$. 

We start by defining:
\begin{align*}
L_{\text{VIC-corr}}(Z,Z')=\lambda s(Z, Z') + \mu [ &v_{\text{corr}}(Z) + v_{\text{corr}}(Z')] + \nu [ c_{\text{corr}}(Z) + c_{\text{corr}}(Z')] \text{, where}
\\
s(Z,Z') &= \frac{1}{n}\sum_{i=1}^{n}\Vert z_{i}-z_{i}' \Vert_{2}^{2},\\
v_{\text{corr}}(Z) &= \frac{1}{d}\sum_{i=1}^{d}\left(1-[\frac{1}{n}ZZ^\top ]_{i,i}\right)^{2},
\\
c_{\text{corr}}(Z)&=\frac{1}{d}\sum_{i\neq j}^{d}[\frac{1}{n}ZZ^\top ]_{i,j}^2
\end{align*}
and $\lambda,\mu, \nu > 0$ are hyper-parameters.  We then define 
$\mathcal{L}_{\text{VIC-corr}}\Bigl(\{x_i\}_{i=1}^{n}, \mathcal{T},f\Bigl)=\mathbb{E}_{Z,Z'}\Bigl[L_{\text{VIC-corr}}(Z,Z')\Bigr]$, $z_i=f\Bigl(T_i(x_i)\Bigl),z_i'=f\Bigl(T'_i(x_i)\Bigl)$, where $\left(T_i, T'_i\right) {\sim} \mathcal{T}\otimes\mathcal{T}$.

An equivalent version of Theorem \ref{main lemma VICReg} holds for $\mathcal{L}_{\text{VIC-corr}}$, which we state below as Theorem \ref{mainLemmaVICRegCorr}. Since we use Theorem \ref{mainLemmaVICRegCorr} as an auxiliary result for proving the optimality of augmentation for the Spectral Contrastive Loss, we formulate a stronger statement, where the loss of the minimizer is also guaranteed to be $0$. However, as we will see, this does not significantly change the proof.

\begin{theorem}[Optimality of Augmentations for VICReg-corr]\label{mainLemmaVICRegCorr}\
Let $f^* = C\,\Phi^\top , C \in \mathbb{R}^{d\times n}$ satisfy Condition \ref{non_redundant_representation} and assume Condition \ref{invertable_gram}.  
Then, $C$ has full rank and any $f$ that is a least norm minimizer of 
$\mathcal{L}_{\text{VIC-corr}}\Bigl(\{\phi(x_i)\}_{i=1}^n,\;\mathcal{T}_\mathcal{H}(C),\;f\Bigr)$
over $\mathbb{R}^d \otimes \mathcal{H}$ satisfies $\mathcal{L}_{\text{VIC-corr}}\Bigl(\{\phi(x_i)\}_{i=1}^n,\;\mathcal{T}_\mathcal{H}(C),\;f\Bigr)=0$ and $f\overset{\text{aff}}{\sim}f^*$.
\end{theorem}

The proof of this theorem is very similar to the proof of Theorem \ref{main lemma VICReg}, and we will just note the exact differences. 

The main difference between the proofs of Theorems 4.2 and \ref{mainLemmaVICRegCorr} lies in Lemma \ref{vicreg_insight}, which is the only place the definition of $v(Z)$ and $c(Z)$ come into play. In the proof of Theorem \ref{mainLemmaVICRegCorr}, Lemma \ref{vicreg_insight} is replaced with the following lemma:

\begin{restatable}{lemma}{vicregCorrInsight}\label{vicreg_corr_insight}
 $L_{\text{VIC-corr}}(Z,Z')=0$ if and only if $Z=Z'$ and $\frac{1}{n}ZZ^\top =I$.
\end{restatable}

In the proof of Theorem \ref{main lemma VICReg}, Lemma \ref{vicreg_insight} is used only in the proof of Proposition \ref{aux_vicreg}, which in the proof of Theorem \ref{mainLemmaVICRegCorr} is replaced with the following proposition. Here, the covariance condition is replaced with a condition on the correlation, and the statement is stronger:

\begin{proposition}\label{aux_vicreg_corr}
  Suppose Condition \ref{invertable_gram} is satisfied. Let $f^*\in\mathbb{R}^d\otimes\mathcal{H}$ be of the form $f^*=C\Phi^\top,C \in \mathbb{R}^{d\times n}$ s.t. $\frac{1}{n}[f^*(\phi_1), \ldots, f^*(\phi_n)][f^*(\phi_1), \ldots, f^*(\phi_n)]^\top =I_d$. Then, if $f$ minimizes $\mathcal{L}_{\text{VIC-corr}}\left(\left\{ \phi_{i}\right\} _{i=1}^{n}, \mathcal{T}_\mathcal{H}(C),f\right)$ with least Hilbert-Schmidt-norm over $\mathbb{R}^d\otimes\mathcal{H}$, $f$ satisfies $\mathcal{L}_{\text{VIC-corr}}\left(\left\{ \phi_{i}\right\} _{i=1}^{n}, \mathcal{T}_\mathcal{H}(C),f\right)=0$ and there exists $Q\in O(d)$ such that $f=Qf^*$.
\end{proposition}

\begin{remark}\label{zero-loss}The statement ``$\mathcal{L}_{\text{VIC-corr}}\left(\left\{ \phi_{i}\right\} _{i=1}^{n}, \mathcal{T}_\mathcal{H}(C),f\right)=0$ and there exists $Q\in O(d)$ such that $f=Qf^*$'' might seem stronger than the statement in Proposition \ref{aux_vicreg}, which only guarantees there exists $Q\in O(d)$ such that $f=Qf^*$. However, the proof of Proposition \ref{aux_vicreg} showed that: 

\[\{QC\Phi^\top : Q\in O(d)\}= \mathcal{F}^\ast, \quad\mathcal{F}^\ast\coloneqq\arg\min_{f\in\mathbb{R}^d\otimes\mathcal{H}}\{\Vert f\Vert_{\text{HS}}: \mathcal{L}_{\text{VIC}}\left(\left\{ \phi_{i}\right\} _{i=1}^{n},\mathcal{T},f\right)=0\}\] 

and therefore because $\mathcal{F}^*$ is not empty, the proof also shows that if $f$ is a least norm minimizer then  $\mathcal{L}_{\text{VIC}}\left(\left\{ \phi_{i}\right\} _{i=1}^{n},\mathcal{T},f\right)=0\}$. Therefore, the stronger statement does not change the rest of the proof. 
\end{remark}

We proceed with detailing the changes that are caused by 
using Lemma \ref{vicreg_corr_insight} instead of Lemma \ref{vicreg_insight}. Namely, we replace eq. \eqref{covariance_a} with:
\begin{gather*}
		n^{-1}AK KA^\top =I \numberthis \label{corr_a}
\end{gather*}
	In addition, we assumed that $\frac{1}{n}[f^*(\phi_1), \ldots, f^*(\phi_n)][f^*(\phi_1), \ldots, f^*(\phi_n)]^\top =I_d$ instead of $\text{cov}(CK)=I$, therefore insread of \eqref{covariance_c} we get:
	\begin{gather*}
		n^{-1}CK KC^\top =I \numberthis \label{corr_c}
	\end{gather*}

    We substitute \eqref{first_equality} into \eqref{corr_a} and get:
\begin{gather*}
	n^{-1}AK C^{T}\left(CKC^{T}\right)^{-1}CK KC^\top \left(CKC^{T}\right)^{-1}CKA^\top =I
\end{gather*}
We substitute \eqref{corr_c} into the above equation and get:

\begin{gather*}
	AK C^{T}\left(CKC^{T}\right)^{-1}\left(CKC^{T}\right)^{-1}CKA^\top =I
\end{gather*}
This is exactly eq. \eqref{orthogonal}, therefore the rest of the proof of Proposition \ref{aux_vicreg_corr} goes exactly as Proposition \ref{aux_vicreg}.

We now finish the proof of Theorem \ref{mainLemmaVICRegCorr}.
\begin{proof}
    
The only difference between Theorem \ref{mainLemmaVICRegCorr} and Proposition \ref{aux_vicreg_corr} is that now instead of assuming  $f^*=C\Phi^\top,C \in \mathbb{R}^{d\times n}$ s.t. $\frac{1}{n}[f^*(\phi_1), \ldots, f^*(\phi_n)][f^*(\phi_1), \ldots, f^*(\phi_n)]^\top =I_d$ we just assume $\text{cov}\left([f^*(\phi_1), \ldots, f^*(\phi_n)]\right)$ it is full rank. The covariance being full rank implies $\frac{1}{n}[f^*(\phi_1), \ldots, f^*(\phi_n)][f^*(\phi_1), \ldots, f^*(\phi_n)]^\top $ is full rank:
\begin{align*}
&rank(\frac{1}{n}[f^*(\phi_1), \ldots, f^*(\phi_n)][f^*(\phi_1), \ldots, f^*(\phi_n)]^\top )\\&=rank([f^*(\phi_1), \ldots, f^*(\phi_n)])
\\&\geq rank([f^*(\phi_1), \ldots, f^*(\phi_n)]H_n)\\&=rank(\text{cov}\left([f^*(\phi_1), \ldots, f^*(\phi_n)]\right)
\end{align*}

since the $\frac{1}{n}[f^*(\phi_1), \ldots, f^*(\phi_n)][f^*(\phi_1), \ldots, f^*(\phi_n)]^\top $ is full rank and symmetric, there is a whitening matrix $W$ s.t. $W\frac{1}{n}[f^*(\phi_1), \ldots, f^*(\phi_n)][f^*(\phi_1), \ldots, f^*(\phi_n)]^\top 
W^\top =\frac{1}{n}[Wf^*(\phi_1), \ldots, Wf^*(\phi_n)][Wf^*(\phi_1), \ldots, Wf^*(\phi_n)]^\top =I$ and $W$ is invertible. From Proposition \ref{aux_vicreg_corr} we get that if $f$ is a least norm minimizer of 
$\mathcal{L}_{\text{VIC}-corr}\Bigl(\{\phi(x_i)\}_{i=1}^n,\;\mathcal{T}_\mathcal{H}(C),\;f\Bigr)$
over $\mathbb{R}^d \otimes \mathcal{H}$, it satisfies $\mathcal{L}_{\text{VIC-corr}}\left(\left\{ \phi_{i}\right\} _{i=1}^{n}, \mathcal{T}_\mathcal{H}(C),f\right)=0$ and $f=QWf^*$. In particular, $f\overset{\text{aff}}{\sim}f^*$, where the affine transformation $QW$ is invertible as a product of two invertible matrices.
\end{proof}

We thereby proved the optimality of augmentations for $\mathcal{L}_{\text{VIC}}$ and we proceed by connecting $\mathcal{L}_{\text{VIC-corr}}$ to $\mathcal{L}_{\text{SCL}}$.

\begin{lemma}\label{equivalence_scl_vicregcorr}
$ \mathcal{L}_{\text{SCL}}\Bigl(\{(\phi_i)\}_{i=1}^n,\;\mathcal{T},\;f\Bigr)=-d\iff\mathcal{L}_{\text{VIC-corr}}\Bigl(\{(\phi_i)\}_{i=1}^n,\;\mathcal{T},\;f\Bigr)=0$
\end{lemma}

\begin{proof}
The necessary and sufficient conditions for $L_{\text{VIC}}(Z,Z')=0$ according to Proposition \ref{vicreg_corr_insight} and the necessary and sufficient conditions on the tightness of the bound in Section \ref{lower_bound}, which implies $L_{\text{SCL}}(Z,Z')=-d\iff Z=Z' \text{ and }\frac{1}{n}ZZ^\top =I \iff L_{\text{VIC-corr}}(Z,Z')=0$. Since $0$ and $-d$ are the lowest possible values the respective losses can take, we get a series of equivalent statements: 
\begin{align*}
&\mathcal{L}_{\text{SCL}}\Bigl(\{(\phi_i)\}_{i=1}^n,\;\mathcal{T},\;f\Bigr)=0\\ &\iff
\forall_{(T_1,T'_1,\ldots, T_n, T'_n)\in\text{Support}(\mathcal{T}^{\otimes 2n})} L_{\text{SCL}}(Z,Z')=-d \\
&\iff\forall_{(T_1,T'_1,\ldots, T_n, T'_n)\in\text{Support}(\mathcal{T}^{\otimes 2n})} L_{\text{VIC-corr}}(Z,Z')=0 \\
&\iff\mathcal{L}_{\text{VIC-corr}}\Bigl(\{(\phi_i)\}_{i=1}^n,\;\mathcal{T},\;f\Bigr)=0
\end{align*}

Where: $Z = [f(T_1(\phi_1)),\ldots, f(T_n(\phi_n))]$ and $Z'= [f(T'_1(\phi_1)),\ldots, f(T'_n(\phi_n))]$.
\end{proof}

Now, we proceed by proving Theorem \ref{main lemma SCLNorm}:
\mainLemmaSCLNorm*

\begin{proof}
Let $f$ be a least norm minimizer of $\mathcal{L}_{\text{SCL}}((\{\phi(x_i)\}_{i=1}^n,\;\mathcal{T}_\mathcal{H}(C),\;f)$. We begin by arguing that $f$ is a least norm minimizer of $\mathcal{L}_{\text{VIC-corr}}(\{\phi(x_i)\}_{i=1}^n,\;\mathcal{T}_\mathcal{H}(C),\;f)$:

Let $f'$ be a least norm minimizer of $\mathcal{L}_{\text{VIC-corr}}(\{\phi(x_i)\}_{i=1}^n,\;\mathcal{T}_\mathcal{H}(C),\;f)$, we know from Theorem \ref{mainLemmaVICRegCorr} that $\mathcal{L}_{\text{VIC-corr}}(\{\phi(x_i)\}_{i=1}^n,\;\mathcal{T}_\mathcal{H}(C),\;f')=0$ and therefore from Lemma \ref{equivalence_scl_vicregcorr} we know $\mathcal{L}_{\text{SCL}}(\{\phi(x_i)\}_{i=1}^n,\;\mathcal{T}_\mathcal{H}(C),\;f')=-d$ and therefore $\mathcal{L}_{\text{SCL}}(\{\phi(x_i)\}_{i=1}^n,\;\mathcal{T}_\mathcal{H}(C),\;f)=-d$, since otherwise $f$ would not minimize $\mathcal{L}_{\text{SCL}}$. We use Lemma \ref{equivalence_scl_vicregcorr} again in the other direction and get that $\mathcal{L}_{\text{VIC-corr}}(\{\phi(x_i)\}_{i=1}^n,\;\mathcal{T}_\mathcal{H}(C),\;f)=0$, therefore $f$ must be a minimizer of $\mathcal{L}_{\text{VIC-corr}}$. Moreover, $f$ must be a least norm minimizer of $\mathcal{L}_{\text{VIC-corr}}$, since any minimizer with a smaller norm would also minimize $ \mathcal{L}_{\text{SCL}}$ (again, by Lemma \ref{equivalence_scl_vicregcorr}), contradicting the assumtion that $f$ is the least norm minimizer of $\mathcal{L}_{\text{SCL}}$.

Theorem \ref{main lemma SCLNorm} follows directly, since by as a least norm minimizer of $\mathcal{L}_{\text{VIC-corr}}\Bigl(\{\phi(x_i)\}_{i=1}^n,\;\mathcal{T}_\mathcal{H}(C),\;f\Bigr)$, Theorem \ref{mainLemmaVICRegCorr} implies $f\overset{\text{aff}}{\sim}f^*$.
\end{proof}

\subsection{Proof of Theorem  \ref{main lemma Barlow Twins} (Barlow Twins)}\label{bt_appendix}
We structure the proof by first proving a set of auxiliary results and then proving Theorem  \ref{main lemma Barlow Twins}.

\begin{restatable}{lemma}{btInterpolator}\label{bt interpolator}
    Let $T:\mathcal{H}\rightarrow\mathcal{H}$, and $\mathcal{T}$ a distribution of that yields $T$ with probability $0.5$ and $I_\mathcal{H}$ with probability $0.5$. Then, $\mathcal{L}_{BT}(\{\phi_i\}_{i=1}^{n}, \mathcal{T}, A\Phi^\top )=0$ if and only if $(2n)^{-1}A\Phi^\top (T\Phi\Phi^\top +\Phi\Phi^\top  T^\top )\Phi A^\top  = I_d$.
\end{restatable}
\begin{proof}
Since $\mathcal{L}_{\text{BT}}$ is non-negative, $\mathcal{L}_{\text{BT}}(\{\phi_i\}_{i=1}^{n}, \mathcal{T}, A\Phi^\top )=0$ if and only if for every $(T_1,T'_1,\ldots, T_n, T'_n)$ in the support of $(\mathcal{T}\otimes\mathcal{T}\mid T\neq T')^{\otimes n}$:
\[
L_{\text{BT}}([A\Phi^\top  T_1(\phi_1),\ldots, A\Phi^\top  T_n(\phi_n) ],[A\Phi^\top  T'_1(\phi_1),\ldots, A\Phi^\top  T'_n(\phi_n) ])=0
\]
Recall the definition of $L_{\text{BT}}$:
\begin{align*}
	L_{\text{BT}}(Z,Z')&= \sum_i^n (1- \mathcal{C}_{ii})^2+ \lambda\sum_i^n (1- \mathcal{C}_{ii})^2 
\end{align*}
Since $\lambda>0$, we get that $L_{\text{BT}}(Z,Z')=0$ if and only if $\mathcal{C}\coloneqq\frac{1}{2n}(ZZ'^\top  + Z'Z^\top )=I$. Plugging $ZZ'^\top =\sum_{i=1}^n (A\Phi^\top  T_i(\phi_1)) (A\Phi^\top  T_i'(\phi_1))^\top $:
\begin{align}
	\frac{1}{2n}\sum_{i=1}^n (A\Phi^\top  T_i(\phi_i)) (A\Phi^\top  T_i'(\phi_i))^\top  + (A\Phi^\top  T'_i(\phi_i)) (A\Phi^\top  T_i(\phi_i))^\top  = I
	\label{barlow_twins_sum} 
\end{align}

Since $\mathcal{T}$ has two augmentations in its support and we draw the augmentations conditioned on $T_i\neq T'_i$, every summand of \eqref{barlow_twins_sum} will be of the following form:
\[
(A\Phi^\top  T(\phi_i)) (A\Phi^\top  I_\mathcal{H}(\phi_i))^\top  + (A\Phi^\top  I_\mathcal{H}(\phi_i)) (A\Phi^\top  T(\phi_i))^\top 
\]
Hence, we write \eqref{barlow_twins_sum} in the following way:
\begin{align*}
	\frac{1}{2n}\sum_{i=1}^n (A\Phi^\top  T(\phi_i)) (A\Phi^\top  I_\mathcal{H}(\phi_i))^\top  + (A\Phi^\top  I_\mathcal{H}(\phi_i)) (A\Phi^\top  T(\phi_i))^\top  &= I\\
	\frac{1}{2n}\sum_{i=1}^n (A\Phi^\top  T\phi_i) (A\Phi^\top  \phi_i)^\top  + (A\Phi^\top  \phi_i) (A\Phi^\top  T \phi_i)^\top  &= I \\
	\frac{1}{2n}\sum_{i=1}^n  A\Phi^\top  T  \phi_i \phi_i^\top  \Phi A^\top  + A\Phi^\top  \phi_i \phi_i^\top  T^\top \Phi A^\top  &= I \\
	\frac{1}{2n}A\Phi^\top  (T (\sum_{i=1}^n  \phi_i \phi_i^\top )  +(\sum_{i=1}^n  \phi_i \phi_i^\top ) T^\top )\Phi A^\top  &= I \\
	\frac{1}{2n}A\Phi^\top  (T \Phi \Phi^\top   + \Phi \Phi^\top  T^\top )\Phi A^\top  &= I
\end{align*}

\end{proof}

\begin{lemma}\label{simon lemma} \citep[See also][Proposition 4.3]{pmlr-v202-simon23a} 
	Let $\Gamma\in\mathbb{R}^{d\times d}$ be a symmetric positive definite matrix. Let $\Gamma=USU^\top $ its spectral decomposition. The set of solutions of $W\Gamma W^\top =I$ is exactly $W=QS^{-\frac{1}{2}}U^\top $ where $Q\in O(d)$.
\end{lemma}
\begin{proof}
	It is easy to verify that $W=QS^{-\frac{1}{2}}U^\top $ where $Q\in O(d)$ indeed satisfies $W\Gamma W^\top =I$. We would like to prove that any $W$ such that $W\Gamma W^\top =I$ is of that form. Denote $\Gamma^\frac{1}{2}= US^\frac{1}{2}U^\top $.
	\begin{gather*}
		(W\Gamma^\frac{1}{2})(W\Gamma^\frac{1}{2})^\top =I
	\end{gather*}
	
	Hence, $W\Gamma^\frac{1}{2}\in O(d)$. Finally:
	\[
	W=W\Gamma^\frac{1}{2}(\Gamma^\frac{1}{2})^{-1}=W\Gamma^\frac{1}{2}US^{-\frac{1}{2}}U^\top 
	\]
	Where $W\Gamma^\frac{1}{2}U\in O(d)$.
	
\end{proof}
\begin{restatable}{lemma}{crazyLemma}\label{crazy_lemma}
Let $C$ be a rank $d$ matrix and $K$ be a positive definite matrix. The set of solutions $A$ to 
	$AKC^\top (CKC^\top )^{-2}CKA^\top =I_d$ is $\{QC: Q\in O(d)\}$ 
\end{restatable}
\begin{proof}
	Recall Equation \eqref{crazy_eq}:
	\begin{gather*}
		AKC^\top (CKC^\top )^{-2}CKA^\top =I_d
	\end{gather*}
	We denote $\Gamma=(CKC^\top )^{-2}$. First, we show that the least Frobenius norm solutions of $W\Gamma W^\top $ are $W=QCKC^\top $.
	
	$CKC^\top $ is a positive definite and symmetric, and $rank(CKC^\top )=d$. This follows from the fact that if $K$ is a positive definite symmetric matrix, and if $K=LL^\top $ is the Cholesky decomposition of $K$, $CKC^\top =(CL)(CL)^\top $ and $rank(CKC^\top )=rank(CL)=rank(C)=d$.
	
	Let $V\Sigma V^\top $ be the spectral decomposition of $CKC^\top $.
	\[
	(CKC^\top )^{-2}=V\Sigma^{-2}V^\top 
	\]
	Therefore, according to Lemma \ref{simon lemma}, the least Frobenius norm solutions of $W\Gamma W^\top $ are exactly $W=U\Sigma V^\top $.
	\[
	W=U\Sigma V^\top = U I_d \Sigma V^\top  = U V^\top  V \Sigma V^\top  = (U V^\top ) C K C^\top . 
	\]
	Where $UV^\top \in O(d)$.
	Consequently, all solutions to \eqref{crazy_eq} are of the form $A=QC$ where $Q\in O(d)$.
	
\end{proof}

\mainLemmaBt*

\begin{proof}
    
We prove a slightly stronger result, namely, instead of proving the equivalence of $f$ and $f^*$ up to an affine transformation, we prove an equivalence up to an orthogonal transformation, meaning there exists $Q\in O(d)$ s.t. $f=Qf^*$.

For simplicity of notation, we denote $\mathcal{T}=\mathcal{T}_\mathcal{H}^\text{BT}(C)$. Define:
\[
\mathcal{F}^\ast\coloneqq\arg\min_{f\in\mathbb{R}^d\otimes\mathcal{H}}\{\Vert f\Vert_{\text{HS}}: \mathcal{L}_{\text{BT}}\left(\left\{ \phi_{i}\right\} _{i=1}^{n},\mathcal{T},f\right)=0\}
\]
We will prove that $\mathcal{F}^*=\{QC\Phi^\top : Q\in O(d)\}$.

We begin by rewriting the kernel Gram matrix $K$ and the Lyapunov equation in Definition \ref{definition bt} with new terms denoted by $\Psi$ and $G$:
\begin{align*}
	BK+KB&=2n\Psi^\top \Phi K^{-1}GK^{-1}\Phi^\top \Psi\\
	K&=\Psi^\top \Phi\Phi^\top \Psi\\
	\Psi&\coloneqq\Phi K^{-\frac{1}{2}} \\
    G&\coloneqq KC^\top (CKC^\top )^{-2}CK
\end{align*}

Notice that $\Phi=\Psi\Psi^\top \Phi$. Multiplying from the left by $\Psi$ and from the right by $\Psi^\top $, we get for the solution $B$ of the Lyapunov equation:
\begin{align*}
	BK+KB&=2n\Psi^\top \Phi K^{-1}GK^{-1}\Phi^\top \Psi\\
	\Psi BK\Psi^\top  +\Psi KB \Psi^\top  &=2n\Psi\Psi^\top \Phi K^{-1}GK^{-1}\Phi^\top \Psi\Psi^\top  \\
	\Psi B\Psi^\top \Phi\Phi^\top \Psi\Psi^\top  +\Psi \Psi^\top \Phi\Phi^\top \Psi B \Psi^\top  &=2n\Psi\Psi^\top \Phi K^{-1}GK^{-1}\Phi^\top \Psi\Psi^\top  \\
	\Psi B\Psi^\top \Phi\Phi^\top  +\Phi\Phi^\top \Psi B \Psi^\top  &=2n\Phi K^{-1}GK^{-1}\Phi^\top  \numberthis\label{main_eq_bt2}
\end{align*}

Recall Lemma \ref{bt interpolator}:
\btInterpolator*

Applying this lemma to $T=\Phi K^{-\frac{1}{2}}B K^{-\frac{1}{2}}\Phi^\top$ (which is the form of the transformation in Definition \ref{definition bt}):
\begin{align*}
& \mathcal{L}_{BT}(\{\phi_i\}_{i=1}^{n}, \mathcal{T}, A\Phi^\top )=0\\&\iff (2n)^{-1}A\Phi^\top (\Phi K^{-\frac{1}{2}}B K^{-\frac{1}{2}}\Phi^\top \Phi\Phi^\top +\Phi\Phi^\top \Phi K^{-\frac{1}{2}}B K^{-\frac{1}{2}}\Phi^\top )\Phi A^\top  = I_d \\
&\iff
(2n)^{-1}A\Phi^\top (\Psi B \Psi^\top  \Phi\Phi^\top +\Phi\Phi^\top \Psi B \Psi^\top )\Phi A^\top  = I_d
\end{align*}
Combining with \eqref{main_eq_bt2} we get:
\begin{align*}
\mathcal{L}_{BT}(\{\phi_i\}_{i=1}^{n}, \mathcal{T}, A\Phi^\top )=0 &\iff (2n)^{-1}A\Phi^\top (2n\Phi K^{-1}GK^{-1}\Phi^\top)\Phi A^\top  = I_d\\
&\iff AGA^\top =I_d
\end{align*}
From the definition of $G$ we get that $A$ is a solution of $\mathcal{L}_{\text{BT}}(\{\phi_i\}_{i=1}^{n}, \mathcal{T}, A\Phi^\top )=0$ if and only if 
\begin{gather}
	AKC^\top (CKC^\top )^{-2}CKA^\top =I_d\label{crazy_eq}
\end{gather}

Recall Lemma \ref{crazy_lemma}:
\crazyLemma*

Therefore, we get the following claim: 

\begin{claim}\label{characterization_of_solutions}
The set of solutions $A$ of $\mathcal{L}_{\text{BT}}(\{\phi_i\}_{i=1}^{n}, \mathcal{T}, A\Phi^\top )=0$ is $\{QC:Q\in O(d)\}$.    
\end{claim}
 Now we would like to use this result to show that the set of least norm minimizers of $\mathcal{L}_{\text{BT}}$ is $\{QC\Phi^\top :Q\in O(d)\}$. The way this follows is simple: If $f$ is a least norm minimizer of $\mathcal{L}_{\text{BT}}$, then by Proposition \ref{representer} it has to be for the form $f=A\Phi^\top $ and therefore because of Claim \ref{characterization_of_solutions} be of the form $f=QC\Phi^\top , Q\in O(d)$. Conversely, if $f=QC\Phi^\top ,Q\in O(d)$, then Claim \ref{characterization_of_solutions} implies it must satisfy $\mathcal{L}_{\text{BT}}(f)=0$, all it is left to prove that it is a least-norm solution, which we do next, and thereby proving the Theorem.

Assume by way of contradiction there is a $f'\in\mathcal{F}^\ast$ such that $\Vert f'\Vert_{\text{HS}}<\Vert f\Vert_{\text{HS}}$. Applying Proposition \ref{representer} and Claim \ref{characterization_of_solutions} we know that $f'=\tilde{Q}'C\Phi^\top $ for $\tilde{Q}'\in O(d)$. We get that $\Vert\tilde{Q}C\Phi^\top \Vert_{\text{HS}}<\Vert\tilde{Q}'C\Phi^\top \Vert_{\text{HS}}$ for two matrices $\tilde{Q},\tilde{Q}'\in O(d)$. This is a contradiction since for every $Q\in O(d)$:
\begin{gather*}
	\Vert QC\Phi^\top \Vert_{\text{HS}}^2=Tr((QC\Phi^\top )^\top (QC\Phi^\top ))=Tr((C\Phi^\top )^\top Q^\top Q(C\Phi^\top ))\\=Tr((C\Phi^\top )^\top (C\Phi^\top ) =  \Vert C\Phi^\top \Vert_{\text{HS}}^2 
\end{gather*}

To summarize, we get:
\[
\{QC\Phi^\top : Q\in O(d)\}= \mathcal{F}^\ast
\]
\end{proof}

\subsection{Proof of Corollary \ref{main proposition}}
\mainProposition*
\begin{proof} The result is a corollary of Theorems 4.2, 4.3 and 4.5, as they imply that for a least norm minimizer of $\mathcal{L}\Bigl(\{\phi(x_i)\}_{i=1}^n,\;\mathcal{T}_\mathcal{H}(FK^{-1}),\;f\Bigr)$ there exist an invertible matrix $M\in\mathbb{R}^{d\times d}$ and a vector $b\in\mathbb{R}^d$ such that:
\[
f=MFK^{-1}\Phi^\top +b
\]
Implying:
\begin{align*}
f\Phi&=MFK^{-1}\Phi^\top \Phi+b \\
f\Phi&=MFK^{-1}K+b\\
f\Phi&=MF+b\\
\forall_{i\in[n]}f(\phi(x_i))&=Mf^*(x_i)+b
\end{align*}

\end{proof}

\section{The Preimage Problem for Kernel Machines}\label{preimage}

We detail closed-form approximation proposed by \citet{honeine_closed-form_2011} to the preimage problem, introduced in Section 5: Given the training data in the input space $X=[x_1,\ldots,x_n] \in \mathbb{R}^{m \times n}$ and a single point $\phi'=\Phi \theta$ in the Hilbert space, we solve the least-squares problem:
\[
\tilde\phi^{-1}(\theta) = \arg\min_{x'}\Bigl\Vert X^\top  x' - (X^\top X-\mu_P K^{-1})\theta\Bigr\Vert^2
\]
where $\mu_P>0$ is a hyperparameter and $K$ is, as usual, the kernel Gram matrix on $X$.

There could potentially be no exact solutions $x'$, one exact solution or multiple solutions; In our experiments, we use:
\[
x'=(X^\top)^+(X^\top X-\mu_P K^{-1})\theta
\]

Where $^+$ denotes the Moore–Penrose inverse.

\section{Details on the Experiments}\label{exp_details}
For all of the experiments, we used either a pretrained ResNet given by  \verb|ResNet50_Weights.IMAGENET1K_V2| in PyTorch \citep{paszke2019pytorchimperativestylehighperformance}, or a pretrained ViT given by \verb|ViT_B_16_Weights.IMAGENET1K_V1|.  We reduced the dimension of the target representations to 64 using PCA. 

For calculating the augmentations in Figure 2a we used 10,000 images of MNIST training data as $\{x_i\}_{i=1}^n$ and we calculated the augmented images of validation data ,which was not part of $\{x_i\}_{i=1}^n$ in Algorithm 1, meaning these are augmentations produced for new images outside of the ``training set".  We used $\sigma=3$ for the RBF kernel and $\lambda_{\text{ridge}}=1, \mu_P=1$ for these experiments ($\mu_P$ being a parameter for solving the preimage problem, see Section \ref{preimage}).

For Figures 2b and 2c and the additional experiments, we used the RBF kernel with $\sigma=1$ and $\lambda=5, \mu=5, \nu=1$ for the VICReg losses. The parameters were chosen based on the empirical speed of convergence to zero Procrustes distance. We chose 10,000 images from each dataset and we optimized the loss as one batch consisting of the training images and their augmented views in the RKHS. We parameterized the learnable function as $f=C\Phi^\top $ where $C\in\mathbb{R}^{d\times n}$ are the learnable parameters, the calculations were done using the kernel trick. We used the Adam optimizer \citep{kingma2017adammethodstochasticoptimization} with a learning rate of 0.001. Each experiment was repeated 3 times, and the error terms correspond to the standard error of the mean.

For the original VICReg loss we used the following definition \citep{bardes_vicreg_2022}:
\begin{gather*}
    L(Z,Z')=\lambda s(Z, Z') + \mu [ v(Z) + v(Z')] + \nu [ c(Z) + c(Z')]\quad\text{, where:}\\
    s(Z,Z') = \frac{1}{n}\sum_{i=1}^{n}\Vert z_{i}-z_{i}' \Vert_{2}^{2},
\quad
v(Z) =\frac{1}{d}\sum_{i=1}^{d}\max(0,1-\sqrt{[\text{cov}\left(Z\right)]_{i,i}+\epsilon}),
\\
c(Z) = \frac{1}{d}\sum_{i\neq j}^{d}[\text{cov}(Z)]_{i,j}^2
\end{gather*}

and we set $\epsilon=0.0001$.

\section{Extension to Neural Networks}
We formalize the argument made in Section 7. Namely we show the following:

\begin{proposition}[Optimal Augmentations Guarantee Global Minimum for Neural Networks]\label{global_min_nn}
Let ${\mathcal{L}\in\{\mathcal{L}_\text{VIC},\mathcal{L}_\text{SCL}\}}$ and $f^* = \Theta^*\phi_{\theta^*}(\cdot)$ where $\phi_\theta:\mathbb{R}^m\to\mathbb{R}^k$ is a function parameterized by $\theta$ and $\Theta\in\mathbb{R}^{d\times k}$. Assume $f^*$ satisfies Condition \ref{non_redundant_representation} and that the rows of $\Theta^*$ are in the span of $\{\phi(x_i)\}_{i=1}^n$. Define a distribution of transformations $\mathcal{T}$ that yields $I_k$ and $\Theta^{*\top}(\Theta^*\Theta^{*\top})^{-1}\Theta^*$ with probability $\frac{1}{2}$ each. Then, there exists $\Theta^*_{\text{aff}}\overset{\text{aff}}{\sim}\Theta^* $ such that $(\theta^*,\Theta^*_{\text{aff}})\in\arg\min_{\theta,\Theta}\mathcal{L}\Bigl(\{\phi_\theta(x_i)\}_{i=1}^n,\;\mathcal{T},\;\Theta\Bigr)$.
\end{proposition}

\begin{proof}
We define the kernel  $k_{\theta^*}(x,x')=\phi_{\theta^*}(x)^\top \phi_{\theta^*}(x')$. Since the rows of $\Theta^*$ are in the span of $\{\phi(x_i)\}_{i=1}^n$, there is a matrix $C$ such that $\Theta^*=C[\phi_{\theta^*}(x_1),\ldots,\phi_{\theta^*}(x_n)]^T$. The corresponding distribution of transformations $\mathcal{T}$ yields $I_k$ and $\Phi\,C^\top \bigl(C\,K\,C^\top \bigr)^{-1}C\,\Phi^\top = \Theta^{*\top}(\Theta^*\Theta^{*\top})^{-1}\Theta^*$ with probability $\frac{1}{2}$ each (Definition \ref{definition vicreg}).

We would like to invoke the proofs of Theorems 4.2 and 4.3 with the kernel $k_{\theta^*}(x,x')=\phi_{\theta^*}(x)^\top \phi_{\theta^*}(x')$ and the target $C$. The proofs of the theorems imply that set of least norm minimizers of $\mathcal{L}\Bigl(\{\phi_{\theta^*}(x_i)\}_{i=1}^n,\;\mathcal{T},\;\Theta\Bigr)$, defined as $\mathcal{F}^\ast$ in the proof, is not empty and its elements are of the form $A\Theta^*$ where $A$ is invertible. Moreover, the minimizers achieve loss $0$ for VICReg and $-d$ for SCL, which are the lowest possible values the losses can take (See Remark \ref{zero-loss} and Lemma \ref{equivalence_scl_vicregcorr}). 

However, these theorems assume that the matrix $K=[k_{\theta^*}(x_i,x_j)]_{i.j}$ is full rank (Condition \ref{invertable_gram}). While this assumption is standard for typical kernels, in the case of neural networks it can be too strict, since usually $n>k$ (with the notable exception of the commonly studied infinitely wide limit of neural networks). Therefore, we begin by relaxing Condition \ref{invertable_gram} of Theorems 4.2 and 4.3.  

For Theorem \ref{main lemma VICReg}, Condition \ref{invertable_gram} is only used at the beginning of the proof of Proposition \ref{aux_vicreg} to prove that $CKC^\top $ is full rank. We show an alternative proof that does not assume Condition \ref{invertable_gram} but only uses Condition \ref{non_redundant_representation}:

The matrix $K$ is psd and therefore there is a $K^\frac{1}{2}$ such that $K=K^{\frac{1}{2}}K^{\frac{1}{2}}$. Condition \ref{non_redundant_representation} states $\text{cov}\left([f^\ast(x_1), \ldots, f^\ast(x_n)]\right)=\frac{1}{n}CKHH^TK^TC^T=\frac{1}{n}(CK^\frac{1}{2})K^\frac{1}{2}HH^TK^TC^T$ is full rank, which implies $rank(CK^{\frac{1}{2}})\geq d$ by the submultiplicativity property of the matrix rank. $C$ is a $d\times n$ matrix and therefore $rank(CK^{\frac{1}{2}})= d$. It is generally the case that the rank of a matrix $A$ is equal to the rank of the matrix $AA^T$ because they share the same number of non-zero singular values, therefore: $rank(CK^{\frac{1}{2}})=rank((CK^{\frac{1}{2}})(CK^{\frac{1}{2}})^T)=rank(CKC^T)=d$.

The proof of Theorem \ref{main lemma SCLNorm} relies on the proof of Theorem \ref{main lemma VICReg} to prove an equivalent result for $\mathcal{L}_{\text{VICReg-corr}}$ (Theorem \ref{mainLemmaVICRegCorr}) and does not use the assumption that $K$ is invertible besides that. Therefore, Condition \ref{invertable_gram} can be relaxed for both theorems. 

To summarize , we get that for the distribution of augmentations $\mathcal{T}$, of every $\Theta$ that minimizes $\mathcal{L}\Bigl(\{\phi_{\theta^*}(x_i)\}_{i=1}^n,\;\mathcal{T},\;\Theta\Bigr)$ with least norm satisfies $\Theta\overset{\text{aff}}{\sim}\Theta^*$ and achieves the lowest possible value $\mathcal{L}$ can take ($\mathcal{L}=0$ for $\mathcal{L}_{\text{VICReg}}$ and $\mathcal{L}=-d$ for $\mathcal{L}_{\text{SCL}}$). Let $\Theta^*_{\text{aff}}\overset{\text{aff}}{\sim}\Theta^*$ be one such minimizer, we get that for every $\theta$ and $\Theta$, $\mathcal{L}\Bigl(\{\phi_{\theta}(x_i)\}_{i=1}^n,\;\mathcal{T},\;\Theta\Bigr)\geq \mathcal{L}\Bigl(\{\phi_{\theta^*}(x_i)\}_{i=1}^n,\;\mathcal{T},\;\Theta^*\Bigr)$.

\end{proof}

\paragraph{The Preimage Problem for Neural Networks}
    In the proof of Proposition \ref{global_min_nn}, we relaxed Condition \ref{invertable_gram} to use Theorems 4.2 and 4.3 for neural networks. However, we also implicitly assumed Condition \ref{invertable_gram} to calculate the preimages in the input space (Section \ref{preimage}). Luckily, for neural networks, there is a natural way to calculate the preimages. Namely, given $\phi'\in\mathbb{R}^k$ the preimage problem can be formulated as $\min_{x\in\mathcal{X}}\Vert\phi_{\theta^*}(x)-\phi'\Vert^2$. For neural networks, the gradient $\nabla_x\phi_{\theta^*}(x)$ can be efficiently calculated. Therefore, gradient descent in the input space can be used to calculate the preimages.

\end{document}